\documentclass[journal]{IEEEtran}
\usepackage{cite}
\usepackage{graphicx}
\usepackage{amsfonts}
\usepackage{amssymb}
\usepackage{amsmath}
\usepackage{bm}
\usepackage{mathtools}
\usepackage{amsthm}
\usepackage{algorithm}
\usepackage{algpseudocode}
\usepackage{array}
\usepackage{subcaption}
\usepackage{stfloats}
\usepackage{url}
\usepackage{xcolor}
\usepackage{bbm}
\usepackage{multirow}
\usepackage{booktabs}
\usepackage{makecell}
\usepackage[linkcolor=black, citecolor=black, urlcolor = blue, hidelinks]{hyperref}

\allowdisplaybreaks

\newtheorem{theorem}{Theorem}
\newtheorem{corollary}{Corollary}
\newtheorem*{remark}{Remark}

\DeclareMathOperator{\tr}{Tr}
\DeclareMathOperator{\E}{\mathbb{E}}

\begin{document}

\title{Robust Bayesian Cluster Enumeration Based on the $t$ Distribution}

\author{Freweyni~K.~Teklehaymanot,
        Michael~Muma,~\IEEEmembership{Member,~EURASIP,}
        and~Abdelhak~M.~Zoubir,~\IEEEmembership{Member,~EURASIP}
\thanks{F. K. Teklehaymanot was with the Signal Processing Group and the Graduate School of Computational Engineering, Technische Universit\"at Darmstadt, Darmstadt, Germany (e-mail: ftekle@spg.tu-darmstadt.de).}
\thanks{M. Muma is with the Signal Processing Group, Technische Universit\"at Darmstadt, Darmstadt, Germany (e-mail: muma@spg.tu-darmstadt.de).}
\thanks{A. M. Zoubir is with the Signal Processing Group and the Graduate School of Computational Engineering, Technische Universit\"at Darmstadt, Darmstadt, Germany (e-mail: zoubir@spg.tu-darmstadt.de).}
}

\maketitle

\begin{abstract}
A major challenge in cluster analysis is that the number of data clusters is mostly unknown
and it must be estimated prior to clustering the observed data. In  real-world  applications, the observed data is often subject to heavy tailed noise and outliers which obscure the true underlying structure of the data.  Consequently, estimating the number of clusters becomes challenging. To this end, we derive a robust cluster enumeration criterion by formulating the problem of estimating the number of clusters as maximization of the posterior probability of multivariate $t_\nu$ distributed candidate models. We utilize Bayes' theorem and asymptotic approximations to come up with a robust criterion that possesses a closed-form expression. Further, we refine the derivation and provide a robust cluster enumeration criterion for data sets with finite sample size. The robust criteria require an estimate of cluster parameters for each candidate model as an input. Hence, we propose a two-step cluster enumeration algorithm that uses the expectation maximization algorithm to partition the data and estimate cluster parameters prior to the calculation of one of the robust criteria. The performance of the proposed algorithm is tested and compared to existing cluster enumeration methods using numerical and real data experiments.   
\end{abstract}


\section{Introduction}
\IEEEPARstart{C}{luster} analysis is an unsupervised learning task that finds the intrinsic structure in a set of unlabeled data by grouping similar objects into clusters. Cluster analysis plays a crucial role in a wide variety of fields of study, such as social sciences, biology, medical sciences, statistics, machine learning, pattern recognition, and computer vision \cite{kaufman1990,king2015,dave1997,xu2005}. A major challenge in cluster analysis is that the number of clusters is usually unknown but it is required to cluster the data. The estimation of the number of clusters, also called cluster enumeration, has attracted interest for decades and various methods have been proposed in the literature, see for example \cite{dunn1973,davies1979,calinski1974,rousseeuw1987,tibshirani2001,pelleg2000,kalogeratos2012,hamerly2003,feng2006,constantinopoulos2006,huang2017,mehrjou2016,krzanowski1988,zhao2014,teklehaymanot2016,teklehaymanot22017,teklehaymanot2018} and the reviews in \cite{arbelaitz2013,milligan1985,maulik2002,halkidi2001,xu2005}. However, to this day, no single best cluster enumeration method exists. 

In  real-world  applications, the observed data is often subject to heavy tailed noise and outliers \cite{dave1997,maria2005,escudero2011,zoubir2012,zoubir2018,binder2016} which obscure the true underlying structure of the data.  Consequently, cluster enumeration becomes even more challenging when either the data is contaminated by a fraction of outliers or there exist deviations from the distributional assumptions. To this end, many robust cluster enumeration methods have been proposed, see
\cite{wang2018,neykov2007,gallegos2009,gallegos2010,fraley1998,dasgupta1998,andrews2012,mcnicholas2012,frigui1996,hu2011,binder2017,escudero2011,wu2009,zemene2016,ott2014,luis2010} and the references therein. A popular approach in robust cluster analysis is to use the Bayesian information criterion (BIC), as derived by Schwarz \cite{schwarz1978}, to estimate the number of data clusters after either removing outliers from the data \cite{neykov2007,gallegos2009,gallegos2010,wang2018}, modeling noise or outliers using an additional component in a mixture model \cite{fraley1998,dasgupta1998}, or exploiting the idea that the presence of outliers causes the distribution of the data to be heavy tailed and, subsequently, modeling the data as a mixture of heavy tailed distributions \cite{andrews2012,mcnicholas2012}. 
For example, modeling the data using a family of $t_\nu$ distributions \cite{mclachlan1998,mclachlan2000,kotz2004,lange1989,liu1995,kibria2006,kent1994} provides a principled way of dealing with outliers by giving them less weight in the objective function. The family of $t_\nu$ distributions is flexible as it contains the heavy tailed Cauchy for the degree of freedom parameter $\nu=1$ and the Gaussian distribution for $\nu\rightarrow \infty$ as special cases. Consequently, we model the clusters using a family of multivariate $t_\nu$ distributions and derive robust cluster enumeration criteria that account for outliers given that $\nu$ is sufficiently small. 

It is known that the original BIC \cite{schwarz1978,cavanaugh1999} penalizes two structurally different models the same way if they have the same number of unknown parameters \cite{djuric1998,stoica2004}. Hence, careful examination of the original BIC is a necessity prior to its application in specific model selection problems \cite{djuric1998}. Following this line of argument, we have recently derived the BIC for cluster analysis by formulating cluster enumeration as maximization of the posterior probability of candidate models \cite{teklehaymanot22017,teklehaymanot2018}. In \cite{teklehaymanot22017}, we showed that the BIC derived specifically for cluster enumeration has a different penalty term compared to the original BIC. However, robustness was not considered in \cite{teklehaymanot22017}, where a family of multivariate Gaussian candidate models were used to derive the criterion, which we refer to as $\text{BIC}_{\mbox{\tiny N}}$.

To the best of our knowledge, this is the first attempt made to derive a robust cluster enumeration criterion by formulating the cluster enumeration problem as maximization of the posterior probability of multivariate $t_\nu$ candidate models. Under some mild assumptions, we derive a robust Bayesian cluster enumeration criterion, $\text{BIC}_{t_\nu}$. We show that $\text{BIC}_{t_\nu}$ has a different penalty term compared to the original BIC $\left(\text{BIC}_{\mbox{\tiny O}t_\nu}\right)$ \cite{schwarz1978,cavanaugh1999}, given that the candidate models in the original BIC are represented by a family of multivariate $t_\nu$ distributions. Interestingly, for $\text{BIC}_{t_\nu}$ both the data fidelity and the penalty terms depend on the assumed distribution for the data, while for the original BIC changes in the data distribution only affect the data fidelity term.
Asymptotically, $\text{BIC}_{t_\nu}$ converges to $\text{BIC}_{\mbox{\tiny O}t_\nu}$. As a result, our derivations also provide a justification for the use of the original BIC with multivariate $t_\nu$ candidate models from a cluster analysis perspective. Further, we refine the derivation of $\text{BIC}_{t_\nu}$ by providing an exact expression for its penalty term. This results in a robust criterion, $\text{BIC}_{\mbox{\tiny F}t_\nu}$, which behaves better than $\text{BIC}_{t_\nu}$ in the finite sample size case and converges to $\text{BIC}_{t_\nu}$ in the asymptotic regime.

In general, BIC based cluster enumeration methods require a clustering algorithm that partitions the data according to the number of clusters specified by each candidate model and provides an estimate of cluster parameters. Hence, we apply the expectation maximization (EM) algorithm to partition the data prior to the calculation of an enumeration criterion, resulting in a two-step approach. The proposed algorithm provides a unified framework for the robust estimation of the number of clusters and cluster memberships. 

The paper is organized as follows. Section~\ref{sec:probForm} formulates the problem of estimating the number of data clusters and Section~\ref{sec:propcrit} introduces the proposed robust cluster enumeration criterion. Section~\ref{sec:proposedAlgorithm} presents the two-step cluster enumeration algorithm. A comparison of different Bayesian cluster enumeration criteria is given in Section~\ref{sec:comp}. A performance evaluation and comparison to existing methods using numerical and real data experiments is provided in Section~\ref{sec:results}. Finally,  concluding remarks are made in Section~\ref{sec:conclusion}. Notably, a detailed proof is provided in Appendix~\ref{app:B}.

{\bf Notation:} lower- and upper-case boldface letters represent column vectors and matrices, respectively; calligraphic letters denote sets with the exception of $\mathcal{L}$ which represents the likelihood function; $\mathbb{R}$, $\mathbb{R}^+$, and $\mathbb{Z}^+$ denote the set of real numbers, the set of positive real numbers, and the set of positive integers, respectively; $p(\cdot)$ and $f(\cdot)$ denote probability mass function and probability density function (pdf), respectively; $\bm{x}\sim t_{\nu}\left(\bm{\mu},\bm{\Psi}\right)$ represents a multivariate $t_\nu$ distributed random variable $\bm{x}$ with location parameter $\bm{\mu}$, scatter matrix $\bm{\Psi}$, and degree of freedom $\nu$; $\bm{x}\sim \mathcal{N}\left(\bm{\mu},\bm{\Sigma}\right)$ represents a Gaussian distributed random variable $\bm{x}$ with mean $\bm{\mu}$ and covariance matrix $\bm{\Sigma}$; $\hat{\bm{\theta}}$ denotes the estimator (or estimate) of the parameter $\bm{\theta}$; iid stands for independent and identically distributed; {$\left(\mathcal{A.}\right)$} denotes an assumption; $\log$ stands for the natural logarithm; $\E$ represents the expectation operator; $\lim$ stands for the limit; $^\top$ represents vector or matrix transpose; $|\cdot|$ denotes the determinant when its argument is a matrix and an absolute value when its argument is scalar; $\otimes$ represents the Kronecker product; $\text{vec}\left(\bm{Y}\right)$ refers to the stacking of the columns of an arbitrary matrix $\bm{Y}$ into a long column vector; $\mathcal{O}(1)$ denotes Landau's term which tends to a constant as the data size goes to infinity; $\bm{I}_r$ stands for an $r\times r$ dimensional identity matrix; $\bm{0}_{r\times r}$ and $\bm{1}_{r\times r}$ represent an $r\times r$ dimensional all zero and all one matrix, respectively; $\#\mathcal{X}$ denotes the cardinality of the set $\mathcal{X}$; $\triangleq$ represents equality by definition; $\equiv$ denotes mathematical equivalence. 

\section{Problem Formulation}
\label{sec:probForm}
Let $\mathcal{X}\triangleq\{\bm{x}_1,\ldots,\bm{x}_N\}\subset\mathbb{R}^{r\times N}$ denote the observed data set which can be partitioned into $K$ independent, mutually exclusive, and non-empty clusters $\{\mathcal{X}_1,\ldots,\mathcal{X}_K\}$. Each cluster $\mathcal{X}_k$, for $k\in\mathcal{K}\triangleq\{1,\ldots,K\}$, contains $N_k$ data vectors that are realizations of iid multivariate random variables with an unkown distribution. Let $\mathcal{M}\triangleq\{M_{L_\mathrm{min}},\ldots,M_{L_\mathrm{max}}\}$ be a family of candidate models, where $L_\mathrm{min}$ and $L_\mathrm{max}$ represent the specified minimum and maximum number of clusters, respectively. Each candidate model $M_l\in\mathcal{M}$, for $l=L_\mathrm{min},\ldots,L_\mathrm{max}$ and $l\in\mathbb{Z}^+$, represents a partition of $\mathcal{X}$ into $l$ clusters with associated cluster parameter matrix $\bm{\Theta}_l=\left[\bm{\theta}_1,\ldots,\bm{\theta}_l\right]$, which lies in a parameter space $\Omega_l\subset\mathbb{R}^{q\times l}$. 
Our research goal is to estimate the number of clusters in $\mathcal{X}$ given $\mathcal{M}$ assuming that
\begin{description}
 \item[$\left(\mathcal{A.}1\right)$] the constraint $L_\mathrm{min}\leq K\leq L_\mathrm{max}$ is satisfied.
\end{description}
The resulting estimator is required to be insensitive to the presence of heavy-tailed noise and outliers. 

\section{Robust Bayesian Cluster Enumeration Criterion}
\label{sec:propcrit}
Given that some mild assumptions are satisfied, we have recently derived a general Bayesian cluster enumeration criterion referred to as $\text{BIC}_{\mbox{\tiny G}}$ \cite{teklehaymanot22017}. For each candidate model $M_l\in\mathcal{M}$, $\text{BIC}_{\mbox{\tiny G}}$ is calculated as
\begin{align}
 \text{BIC}_{\mbox{\tiny G}}(M_l) & \approx \log p(M_l) + \log f(\hat{\bm{\Theta}}_l|M_l) + \log \mathcal{L}(\hat{\bm{\Theta}}_l|\mathcal{X}) \nonumber \\
 & + \frac{lq}{2}\log 2\pi - \frac{1}{2}\sum_{m=1}^l\log \left|\hat{\bm{J}}_m\right| - \log f(\mathcal{X}),
 \label{eq:finalposterior}
 \end{align}
where $p(M_l)$ denotes a discrete prior on $M_l\in\mathcal{M}$, $f(\hat{\bm{\Theta}}_l|M_l)$ represents the prior on $\hat{\bm{\Theta}}_l$ given $M_l$, $\mathcal{L}(\hat{\bm{\Theta}}_l|\mathcal{X})$ denotes the likelihood function, $f(\mathcal{X})$ represents the pdf of the data set $\mathcal{X}$, and 
\begin{equation}
 \hat{\bm{J}}_m = -\frac{d^2\log \mathcal{L}(\bm{\theta}_m|\mathcal{X}_m)}{d\bm{\theta}_md\bm{\theta}_m^\top}\bigg|_{\bm{\theta}_m=\hat{\bm{\theta}}_m} 
\end{equation}
is the Fisher information matrix (FIM) of the data vectors that belong to the $m$th partition. Interestingly, Eq.~\eqref{eq:finalposterior} gives us the freedom to choose models with desired properties given that those models satisfy some mild assumptions (see \cite{teklehaymanot22017} for details on the assumptions). 

Our objective is to create a robust estimator of the number of clusters. One way to achieve this is to model the data with a family of multivariate $t_\nu$ distributions, where $\nu$ is the degree of freedom parameter. Assuming that 
\begin{description}
 \item[$\left(\mathcal{A.}2\right)$] the degree of freedom parameter $\nu_m$, for $m=1,\ldots,l$, is fixed at some prespecified value,
\end{description}
one can easily show that the $t_\nu$ distribution satisfies all of the necessary assumptions which much be satisfied to reach at  Eq.~\eqref{eq:finalposterior}.

\section{Proposed Robust Bayesian Cluster Enumeration Algorithm}
\label{sec:proposedAlgorithm}
We propose a robust cluster enumeration algorithm to estimate the number of clusters in the data set $\mathcal{X}$. The presented two-step approach utilizes an unsupervised learning algorithm to partition $\mathcal{X}$ into the number of clusters specified by each candidate model $M_l\in\mathcal{M}$ prior to the computation of one of the proposed robust cluster enumeration criteria for that particular model. Note that, in the rest of the manuscript, a candidate model is represented by a multivariate $t_\nu$ distribution. 

\subsection{Proposed Robust Bayesian Cluster Enumeration Criteria}
\label{subsec:proposedCriterion}
For each candidate model $M_l\in\mathcal{M}$, let there be a clustering algorithm that partitions $\mathcal{X}$ into $l$ clusters and provides parameter estimates $\hat{\bm{\theta}}_m =[\hat{\bm{\mu}}_m,\hat{\bm{\Psi}}_m]^\top$, for $m=1,\ldots,l$. Assume that $\left(\mathcal{A.}1\right)$ - $\left(\mathcal{A.}3\right)$ and assumptions ({\bf{A.2}}) - ({\bf{A.5}}) in \cite{teklehaymanot22017} are fulfilled. 
\begin{theorem}
The posterior probability of $M_l$ given $\mathcal{X}$ can be asymptotically approximated by
\begin{equation}
\boxed{
 \begin{array}{rcl}
  \text{\emph{BIC}}_{t_\nu}(M_l) &\triangleq& \log p(M_l|\mathcal{X})\\
  &\approx & \log\mathcal{L}(\hat{\bm{\Theta}}_l|\mathcal{X}) -\frac{q}{2}\sum_{m=1}^l\log \epsilon,
 \end{array}}
   \label{eq:BICt}
\end{equation}
where $q=\frac{1}{2}r(r+3)$ represents the number of estimated parameters per cluster and 
 \begin{equation}
  \epsilon=\text{\emph{max}}\left(\sum_{\bm{x}_n\in\mathcal{X}_m}\!\!w_n^2\,\,,\,\, N_m\right).
  \label{eq:epsilon}
 \end{equation}
$w_n$ is given by Eq.~\eqref{eq:weight} and $N_m=\#\mathcal{X}_m$. The log-likelihood function, also called the data fidelity term, is given by
 \begin{align}
  \log\mathcal{L}(\hat{\bm{\Theta}}_l|\mathcal{X})  &\approx\sum\limits_{m=1}^lN_m\log N_m - \sum\limits_{m=1}^l\frac{N_m}{2}\log|\hat{\bm{\Psi}}_m| \nonumber  \\
  & + \sum\limits_{m=1}^lN_m\log\frac{\Gamma\left((\nu_m+r)/2\right)}{\Gamma\left(\nu_m/2\right)(\pi\nu_m)^{r/2}} \nonumber \\
  & -\frac{1}{2}\sum\limits_{m=1}^l\sum_{\bm{x}_n\in\mathcal{X}_m}(\nu_m+r)\log\left(1+\frac{\delta_n}{\nu_m}\right), 
 \end{align}
where $\Gamma(\cdot)$ denotes the gamma function and $\delta_n=(\bm{x}_n-\hat{\bm{\mu}}_m)^\top\hat{\bm{\Psi}}_m^{-1}(\bm{x}_n-\hat{\bm{\mu}}_m)$ is the squared Mahalanobis distance. The second term in the second line of Eq. \eqref{eq:BICt} is referred to as the penalty term. 
\end{theorem}
\begin{proof}
Proving that Eq.~\eqref{eq:finalposterior} reduces to Eq.~\eqref{eq:BICt} requires approximating $|\hat{\bm{J}}_m|$ and, consequently, writing a closed-form expression for $\text{BIC}_{t_\nu}(M_l)$. A detailed proof is given in Appendix~\ref{app:B}. 
\end{proof}
Once $\text{{BIC}}_{t_\nu}(M_l)$ is computed for each candidate model $M_l\in\mathcal{M}$, the number of clusters in $\mathcal{X}$ is estimated as
\begin{equation}
 \hat{K}_{\text{{BIC}}_{t_\nu}} = \underset{l=L_\mathrm{min},\ldots,L_\mathrm{max}}{\arg \max}\text{{BIC}}_{t_\nu}(M_l).
 \label{eq:BICtK}
\end{equation}

\begin{corollary}
When the data size is finite, one can opt to compute $\log |\hat{\bm{J}}_m|$, without asymptotic approximations to obtain a more accurate penalty term. In such cases, the posterior probability of $M_l$ given $\mathcal{X}$ becomes
\begin{equation}
\boxed{
\text{\emph{BIC}}_{\mbox{\tiny \emph{F}}t_\nu}(M_l) \approx  \log\mathcal{L}(\hat{\bm{\Theta}}_l|\mathcal{X}) -\frac{1}{2}\sum_{m=1}^l\log |\hat{\bm{J}}_m|,
}
\label{eq:BICtF}
\end{equation}
where the expression for $|\hat{\bm{J}}_m|$ is given in Appendix~\ref{app:C}.
\end{corollary}

\subsection{The Expectation Maximization (EM) Algorithm for $t_\nu$ Mixture Models}
We consider maximum likelihood estimation of the parameters of the $l$-component mixture of $t_\nu$ distributions 
\begin{equation}
 f(\bm{x}_n|M_l,\bm{\Phi}_l) = \sum_{m=1}^l\tau_m g(\bm{x}_n;\bm{\mu}_m,\bm{\Psi}_m,\nu_m), 
 \label{eq:mix_t}
\end{equation}
where $g(\bm{x}_n;\bm{\mu}_m,\bm{\Psi}_m,\nu_m)$ denotes the $r$-variate $t_\nu$ pdf and $\bm{\Phi}_l=\left[\bm{\tau}_l,\bm{\Theta}_l^\top,\bm{\nu}_l\right]$. $\bm{\tau}_l=\left[\tau_1,\ldots,\tau_l\right]^\top$ are the mixing coefficients and $\bm{\nu}_l=\left[\nu_1,\ldots,\nu_l\right]^\top$ are assumed to be known or estimated, e.g. using \cite{mclachlan2000}. The mixing coefficients satisfy the constraints $0<\tau_m<1$ for $m=1,\ldots,l$, and $\sum_{m=1}^l\tau_m=1$. 

The EM algorithm is widely used to estimate the parameters of the $l$-component mixture of $t_\nu$ distributions \cite{mclachlan2000,mclachlan1998,kotz2004,nadarajah2008}. The EM algorithm contains two basic steps, namely the E step and the M step, which are performed iteratively until a convergence condition is satisfied. The E step computes
\begin{align}
 \hat{\upsilon}_{nm}^{(i)} &= \frac{\hat{\tau}_m^{(i-1)}g(\bm{x}_n;\bm{\mu}_m^{(i-1)},\bm{\Psi}_m^{(i-1)},\nu_m)}{\sum_{j=1}^l\hat{\tau}_m^{(i-1)}g(\bm{x}_n;\bm{\mu}_j^{(i-1)},\bm{\Psi}_j^{(i-1)},\nu_j)} \label{eq:upsilon}\\
 \hat{w}_{nm}^{(i)} &= \frac{\nu_m+r}{\nu_m +\delta_n^{(i-1)}},\label{eq:w_nm}
\end{align}
where $\hat{\upsilon}_{nm}^{(i)}$ is the posterior probability that $\bm{x}_n$ belongs to the $m$th cluster at the $i$th iteration and $\hat{w}_{nm}^{(i)}$ is the weight given to $\bm{x}_n$ by the $m$th cluster at the $i$th iteration. Once $\hat{\upsilon}_{nm}^{(i)}$ and $\hat{w}_{nm}^{(i)}$ are calculated, the M step updates cluster parameters as follows:
\begin{align}
 \hat{\tau}_m^{(i)} &= \frac{\sum_{n=1}^N\hat{\upsilon}_{nm}^{(i)}}{N}\label{eq:tauhat} \\
 \hat{\mu}_m^{(i)} &= \frac{\sum_{n=1}^N\hat{\upsilon}_{nm}^{(i)}w_{nm}^{(i)}\bm{x}_n}{\sum_{n=1}^N\hat{\upsilon}_{nm}^{(i)}w_{nm}^{(i)}}\label{eq:muhat} \\
 \hat{\bm{\Psi}}_m^{(i)} &= \frac{\sum_{n=1}^N\hat{\upsilon}_{nm}^{(i)}w_{nm}^{(i)}(\bm{x}_n-\hat{\bm{\mu}}_m^{(i)})(\bm{x}_n-\hat{\bm{\mu}}_m^{(i)})^\top}{\sum_{n=1}^N\hat{\upsilon}_{nm}^{(i)}}\label{eq:Psihat}
\end{align}
Algorithm~\ref{alg:BICt} summarizes the working principle of the proposed robust two-step cluster enumeration approach. Given that the degree of freedom parameter $\nu$ is fixed at some finite value, the computational complexity of Algorithm~\ref{alg:BICt} is the sum of the run times of the two steps. Since the initialization, i.e., the K-medians algorithm is performed only for a few iterations, the computational complexity of the first step is dominated by the EM algorithm and it is given by $\mathcal{O}(Nr^2li_{\mathrm{max}})$ for a single candidate model $M_l$, where $i_{\mathrm{max}}$ is a fixed stopping threshold of the EM algorithm. The computational complexity of $\text{BIC}_{t_\nu}(M_l)$ is $\mathcal{O}(Nr^2)$, which is much smaller than the run-time of the EM algorithm and, as a result, it can easily be ignored in the run-time analysis of the proposed algorithm. 
Hence, the total computational complexity of Algorithm~\ref{alg:BICt} is $\mathcal{O}(Nr^2(L_{\mathrm{min}}+\ldots+L_{\mathrm{max}})i_{\mathrm{max}})$.

Note that if $\text{BIC}_{\mbox{\tiny F}t_\nu}(M_l)$ is used in Algorithm~\ref{alg:BICt} instead of $\text{BIC}_{t_\nu}(M_l)$, the computational complexity of the algorithm increases significantly with the increase in the number of features $(r)$ due to the calculation of Eq.~\eqref{eq:est_J}.  
\begin{algorithm}[htb]
 \caption{Robust two-step cluster enumeration approach}
 \begin{algorithmic}
  \State \textit{Inputs:} $\mathcal{X}$, $L_\mathrm{min}$, $L_\mathrm{max}$, and $\nu$
  \For {$l=L_\mathrm{min},\ldots,L_\mathrm{max}$}
  \State {\it Step 1:} model-based clustering
  \State {\it Step 1.1:} the EM algorithm
  \For {$m=1,\ldots,l$}
  \State Initialize $\hat{\bm{\mu}}_m^{0}$ using the K-medians algorithm
  \State Initialize $\hat{\bm{\Psi}}_m^{0}$ using the sample covariance estimator
  \State $\hat{\tau}_m^{0} = \frac{N_m}{N}$
 \EndFor
 \For {$i=1,2,\ldots,i_{\mathrm{max}}$}
 \State \textit{E step:}
 \For {$n=1,\ldots,N$}
 \For {$m=1,\ldots,l$}
 \State Calculate $\hat{\upsilon}_{nm}^{(i)}$ using Eq.~\eqref{eq:upsilon}
 \State Calculate $\hat{w}_{nm}^{(i)}$ using Eq.~\eqref{eq:w_nm}
 \EndFor
 \EndFor
 \State \textit{M step:}
 \For {$m=1,\ldots,l$}
 \State Determine $\hat{\tau}_m^{(i)}$, $\hat{\bm{\mu}}_m^{(i)}$, and $\hat{\bm{\Psi}}_m^{(i)}$ via Eqs.~(\ref{eq:tauhat})-(\ref{eq:Psihat})
 \EndFor
 \State Check for the convergence of either $\hat{\bm{\Phi}}_l^{(i)}$ or $\log \mathcal{L}(\hat{\bm{\Phi}}_l^{(i)}|\mathcal{X})$
 \If {convergence condition is satisfied}
 \State Exit for loop
 \EndIf
 \EndFor
 \State \textit{Step 1.2:} hard clustering
 \For {$n=1,\ldots,N$}
 \For {$m=1,\ldots,l$}
  \begin{equation*}
  \iota_{nm} = \begin{cases}
                               1, & m=\underset{j=1,\ldots,l}{\arg\max}\!\!\!\!\quad\hat{\upsilon}_{nj}^{(i)} \\
                               0, & \text{otherwise}
                      \end{cases}
 \end{equation*}
 \EndFor
 \EndFor
 \For {$m=1,\ldots,l$}
 \State $N_m = \sum_{n=1}^N\iota_{nm}$ 
 \EndFor
 \State {\it Step 2:} calculate $\text{BIC}_{t_\nu}(M_l)$ using Eq.~\eqref{eq:BICt}
 \EndFor
 \State Estimate the number of clusters, $\hat{K}_{\text{BIC}_{t_\nu}}$, in $\mathcal{X}$ via Eq.~\eqref{eq:BICtK}
 \end{algorithmic}
 \label{alg:BICt}
\end{algorithm}

\section{Comparison of Different Bayesian Cluster Enumeration Criteria}
\label{sec:comp}
Model selection criteria that are derived by maximizing the posterior probability of candidate models given data are known to have a common form \cite{stoica2004,rao1989,teklehaymanot22017} that is consistent with 
\begin{equation}
 \log \mathcal{L}(\hat{\bm{\Theta}}_l|\mathcal{X}) - \eta,
 \label{eq:com_form_BIC}
\end{equation}
where $\log \mathcal{L}(\hat{\bm{\Theta}}_l|\mathcal{X})$ is the data fidelity term and $\eta$ is the penalty term. The proposed robust cluster enumeration criteria, $\text{BIC}_{t_\nu}$ and $\text{BIC}_{\mbox{\tiny F}t_\nu}$, and the original BIC with multivariate $t_\nu$ candidate models, $\text{BIC}_{\mbox{\tiny O}t_\nu}$, \cite{andrews2012,mcnicholas2012} have an identical data fidelity term. The difference in these criteria lies in their penalty terms, which are given by
\begin{align}
 \text{BIC}_{t_\nu}:& \quad \eta = \frac{q}{2}\sum_{m=1}^l\log \epsilon \label{eq:pen_t} \\
 \text{BIC}_{\mbox{\tiny F}t_\nu}:& \quad \eta = \frac{1}{2}\sum_{m=1}^l\log|\hat{\bm{J}}_m|\label{eq:pen_Ft} \\
 \text{BIC}_{\mbox{\tiny O}t_\nu}:& \quad \eta = \frac{ql}{2}\log N \label{eq:pen_Ot}, 
\end{align}
where $\epsilon$ and $|\hat{\bm{J}}_m|$ are given by Eq.~\eqref{eq:epsilon} and Eq.~\eqref{eq:est_J}, respectively. Note that $\text{BIC}_{\mbox{\tiny F}t_\nu}$ calculates an exact value of the penalty term, while $\text{BIC}_{t_\nu}$ and $\text{BIC}_{\mbox{\tiny O}t_\nu}$ compute their asymptotic approximation. In the finite sample regime, the penalty term of $\text{BIC}_{\mbox{\tiny F}t_\nu}$ is stronger than the penalty term of $\text{BIC}_{t_\nu}$, while asymptotically all three criteria have an identical penalty term.  
\begin{remark}
 When the degree of freedom parameter $\nu\rightarrow\infty$, $\text{\emph{BIC}}_{t_\nu}$ converges to 
 \begin{align}
  \text{\emph{BIC}}_{\mbox{\tiny \emph{N}}}(M_l) & \approx \sum_{m=1}^l N_m\log N_m - \sum_{m=1}^l\frac{N_m}{2}\log |\hat{\bm{\Sigma}}_m| \nonumber \\
  & - \frac{q}{2}\sum_{m=1}^l\log N_m,
  \label{eq:BICN}
 \end{align}
where $\text{\emph{BIC}}_{\mbox{\tiny \emph{N}}}$ is the asymptotic criterion derived in \cite{teklehaymanot22017} assuming a family of Gaussian candidate models. $\hat{\bm{\Sigma}}_m$ is an estimate of the covariance matrix of the $m$th cluster.  
\end{remark}
\begin{remark}
 A modification of the data distribution of the candidate models only affects the data fidelity term of the original BIC \cite{schwarz1978,cavanaugh1999}. However, given that the BIC is specifically derived for cluster analysis, we showed that both the data fidelity and penalty terms change as the data distribution of the candidate models changes, see Eq.~\eqref{eq:BICt} and Eq.~\eqref{eq:BICN}.  
\end{remark}
A related robust cluster enumeration method that uses the original BIC to estimate the number of clusters is the trimmed BIC (TBIC) \cite{neykov2007}. The TBIC estimates the number of clusters using Gaussian candidate models after trimming some percentage of the data. In TBIC, the fast trimmed likelihood estimator (FAST-TLE) is used to obtain maximum likelihood estimates of cluster parameters. The FAST-TLE is computationally expensive since it carries out a trial and a refinement step multiple times, see \cite{neykov2007} for details.

\section{Experimental Results}
\label{sec:results}
In this section, we compare the performance of the proposed robust two-step algorithm with state-of-the-art cluster enumeration methods using numerical and real data experiments. In addition to the methods discussed in Section \ref{sec:comp}, we compare our cluster enumeration algorithm with the gravitational clustering (GC) \cite{binder2017} and the X-means \cite{pelleg2000} algorithm. All experimental results are an average of $300$ Monte Carlo runs. The degree of freedom parameter is set to $\nu=3$ for all methods that have multivariate $t_\nu$ candidate models. We use the author's implementation of the gravitational clustering algorithm \cite{binder2017}. For the TBIC, we trim $10\%$ of the data and perform $10$ iterations of the trial and refinement steps. For the model selection based methods, the minimum and maximum number of clusters is set to $L_{\mathrm{min}}=1$ and $L_{\mathrm{max}}=2K$, where $K$ denotes the true number of clusters in the data under consideration.  

\subsection{Performance Measures}
Performance is assessed in terms of the empirical probability of detection $(p_{\text{det}})$ and the mean absolute error (MAE), which are defined as 
\begin{align}
 p_{\text{det}} &=\frac{1}{I} \sum_{i=1}^I\mathbbm{1}_{\{\hat{K}_i=K\}} \\
 \textrm{MAE} &= \frac{1}{I}\sum_{i=1}^I|K-\hat{K}_i|,
\end{align}
where $I$ is the total number of Monte Carlo experiments, $\hat{K}_i$ is the estimated number of clusters in the $i$th Monte Carlo experiment, and $\mathbbm{1}_{\{\hat{K}_i=K\}}$ is the indicator function defined as
\begin{equation}
 \mathbbm{1}_{\{\hat{K}_i=K\}} \triangleq 
 \begin{cases}
  1, \quad \textrm{if} \quad \hat{K}_i=K \\
  0, \quad \text{otherwise}
 \end{cases}.
\end{equation} 

In addition to these two performance measures, we also report the empirical probability of underestimation and overestimation, which are defined as 
\begin{align}
 p_{\text{under}} &=\frac{1}{I} \sum_{i=1}^I\mathbbm{1}_{\{\hat{K}_i<K\}} \\
 p_{\text{over}} &=1 - p_{\text{det}} - p_{\text{under}},
\end{align}
respectively. 

\subsection{Numerical Experiments}
\label{subsec:numexp}
\subsubsection{Analysis of the sensitivity of different cluster enumeration methods to outliers}
we generate two data sets which contain realizations of $2$-dimensional random variables $\bm{x}_k\sim\mathcal{N}\left(\bm{\mu}_k,\bm{\Sigma}_k\right)$, where $k=1,2,3$, with cluster centroids $\bm{\mu}_1=[0,5]^\top$, $\bm{\mu}_2=[5,0]^\top$, $\bm{\mu}_3=[-5,0]^\top$, and covariance matrices
\[
\bm{\Sigma}_1= 
 \begin{bmatrix}
  2 & 0.5 \\
  0.5 & 0.5
 \end{bmatrix}\!,
 \bm{\Sigma}_{2}= 
 \begin{bmatrix}
  1 & 0 \\
  0 & 0.1
 \end{bmatrix}\!,
 \bm{\Sigma}_3= 
 \begin{bmatrix}
  2 & -0.5 \\
  -0.5 & 0.5
 \end{bmatrix}\!.
\]
The first data set (Data-1), as depicted in Fig.~\ref{fig:Data-1-1}, replaces a randomly selected data point with an outlier that is generated from a uniform distribution over the range $\left[-20, 20\right]$ on each variate at each iteration. The sensitivity of different cluster enumeration methods to a single replacement outlier over $100$ iterations as a function of the number of data vectors per cluster $(N_k)$ is displayed in Table~\ref{tab:sensitivity}. Among the compared methods, our robust criterion $\text{BIC}_{\mbox{\tiny F}t_3}$ has the best performance in terms of both $p_{\text{det}}$ and MAE. Except for $\text{BIC}_{\mbox{\tiny F}t_3}$ and the TBIC, the performance of all methods deteriorates when $N_k$, for $k=1,2,3$, is small and, notably, $\text{BIC}_{t_3}$ performs poorly. This behavior is attributed to the fact that $\text{BIC}_{t_3}$ is an asymptotic criterion and in the small sample regime its penalty term becomes weak which results in an increase in the empirical probability of overestimation. An illustrative example of the sensitivity of $\text{BIC}_{\mbox{\tiny F}t_3}$ and $\text{BIC}_{\mbox{\tiny N}}$ to the presence of an outlier is displayed in Fig.~\ref{fig:BIC_sensitivity}. Despite the difference in $N_k$, when the outlier is either in one of the clusters or very close to one of the clusters, both $\text{BIC}_{\mbox{\tiny F}t_3}$ and $\text{BIC}_{\mbox{\tiny N}}$ are able to estimate the correct number of clusters reasonably well. The difference between these two methods arises when the outlier is far away from the bulk of data. While $\text{BIC}_{\mbox{\tiny F}t_3}$ is still able to estimate the correct number of clusters, $\text{BIC}_{\mbox{\tiny N}}$ starts to overestimate the number of clusters.   
\begin{figure}[tb]
 \centering
  \begin{subfigure}[b]{0.45\linewidth}
  \centering
  \includegraphics[width=\linewidth]{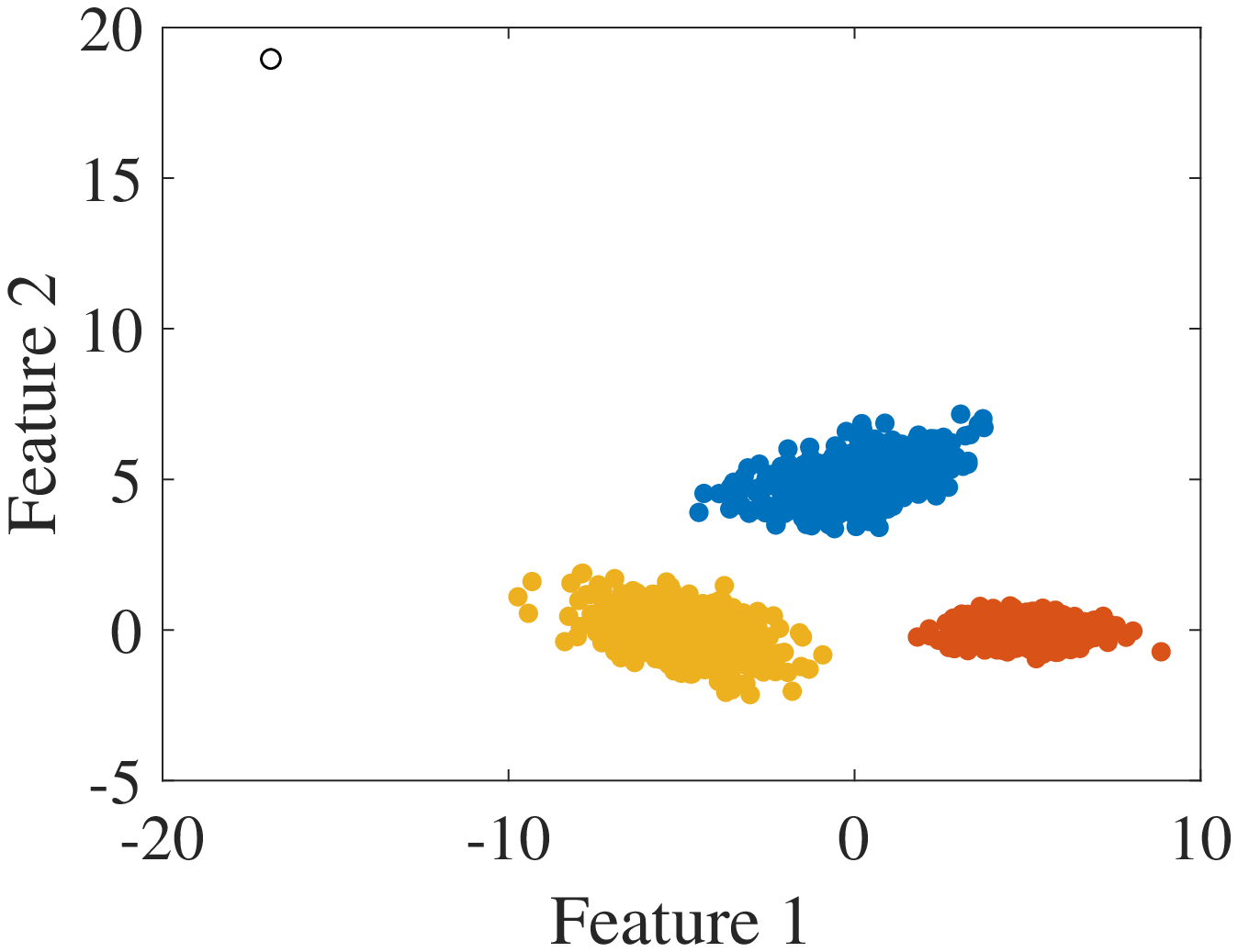}
  \caption{Data-1}
  \label{fig:Data-1-1}
 \end{subfigure}%
   \begin{subfigure}[b]{0.45\linewidth}
  \centering
  \includegraphics[width=\linewidth]{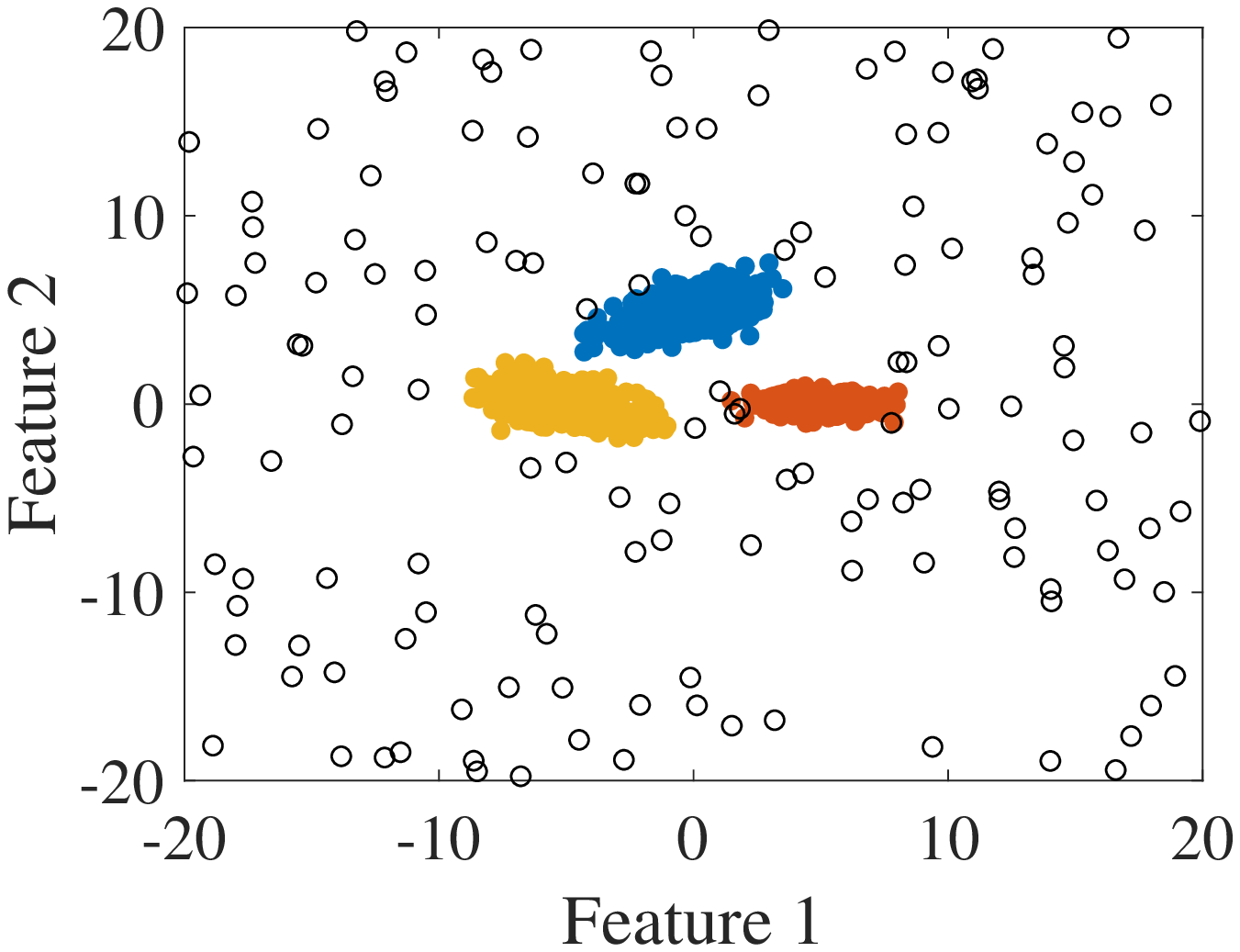}
  \caption{Data-2 with $\alpha=10\%$}
  \label{fig:Data-2-2}
 \end{subfigure}
\caption{Data-1 and Data-2 with $\alpha=10\%$, where filled circles represent clean data and an open circle denotes an outlier.}
\end{figure}

\begin{table}[tb]
 \centering
 \caption{The sensitivity of different cluster enumeration methods to the presence of a single replacement outlier as a function of the number of data points per cluster.}
 \begin{tabular}{cccccc}
  \toprule
   $N_k$ & & $50$ & $100$ & $250$ & $500$ \\
  \midrule
  \multirow{2}{*}{$\text{BIC}_{t_3}$} & $p_{\text{det}}$  & $43.20$ & $92.18$ & $99.77$ & $\bm{100}$ \\
  & MAE & $1.28$ & $0.11$ & $0.002$ & $\bm{0}$ \\
  \cmidrule{2-6}
  \multirow{2}{*}{$\text{BIC}_{\mbox{\tiny F}t_3}$} & $p_{\text{det}}$   & $98.66$ & $\bm{99.746}$ & $\bm{100}$ & $\bm{100}$ \\
  & MAE & $0.0277$ & $\bm{0.0062}$ & $\bm{0}$ & $\bm{0}$ \\
  \midrule
  \multirow{2}{*}{$\text{BIC}_{\mbox{\tiny O}t_3}$ \cite{andrews2012,mcnicholas2012}}  & $p_{\text{det}}$  & $88.13$ & $99.5$ & $99.98$ & $\bm{100}$ \\
  & MAE & $0.18$ & $0.005$ & $0.0002$ & $\bm{0}$ \\
  \cmidrule{2-6}
  \multirow{2}{*}{$\text{TBIC}$ \cite{neykov2007}} & $p_{\text{det}}$  & $\bm{98.75}$ & $99.26$ & $98.92$ & $98.8$ \\
  & MAE & $\bm{0.013}$ & $0.008$ & $0.01$ & $0.01$ \\
  \cmidrule{2-6}
  \multirow{2}{*}{GC \cite{binder2017}} & $p_{\text{det}}$ &  $73.07$ & $94.85$ & $99.80$ & $\bm{100}$ \\
  & MAE & $0.29$ & $0.05$ & $0.002$ & $\bm{0}$ \\
  \midrule
  \multirow{2}{*}{$\text{BIC}_{\mbox{\tiny N}}$ \cite{teklehaymanot22017}} & $p_{\text{det}}$  & $10.92$ & $15.60$ & $33.82$ & $42.20$ \\
  & MAE & $1.25$ & $1.13$ & $0.99$ & $0.83$ \\
  \cmidrule{2-6}
  \multirow{2}{*}{X-means \cite{pelleg2000}} & $p_{\text{det}}$  & $1.24$ & $1.17$ & $1.38$ & $0.17$ \\
  & MAE & $2.69$ & $2.67$ & $2.33$ & $2.13$ \\
  \bottomrule
 \end{tabular}
\label{tab:sensitivity}
\end{table}

\begin{figure}[tb]
 \centering
 \begin{subfigure}[h]{0.45\linewidth}
 \centering
  \includegraphics[width=\linewidth]{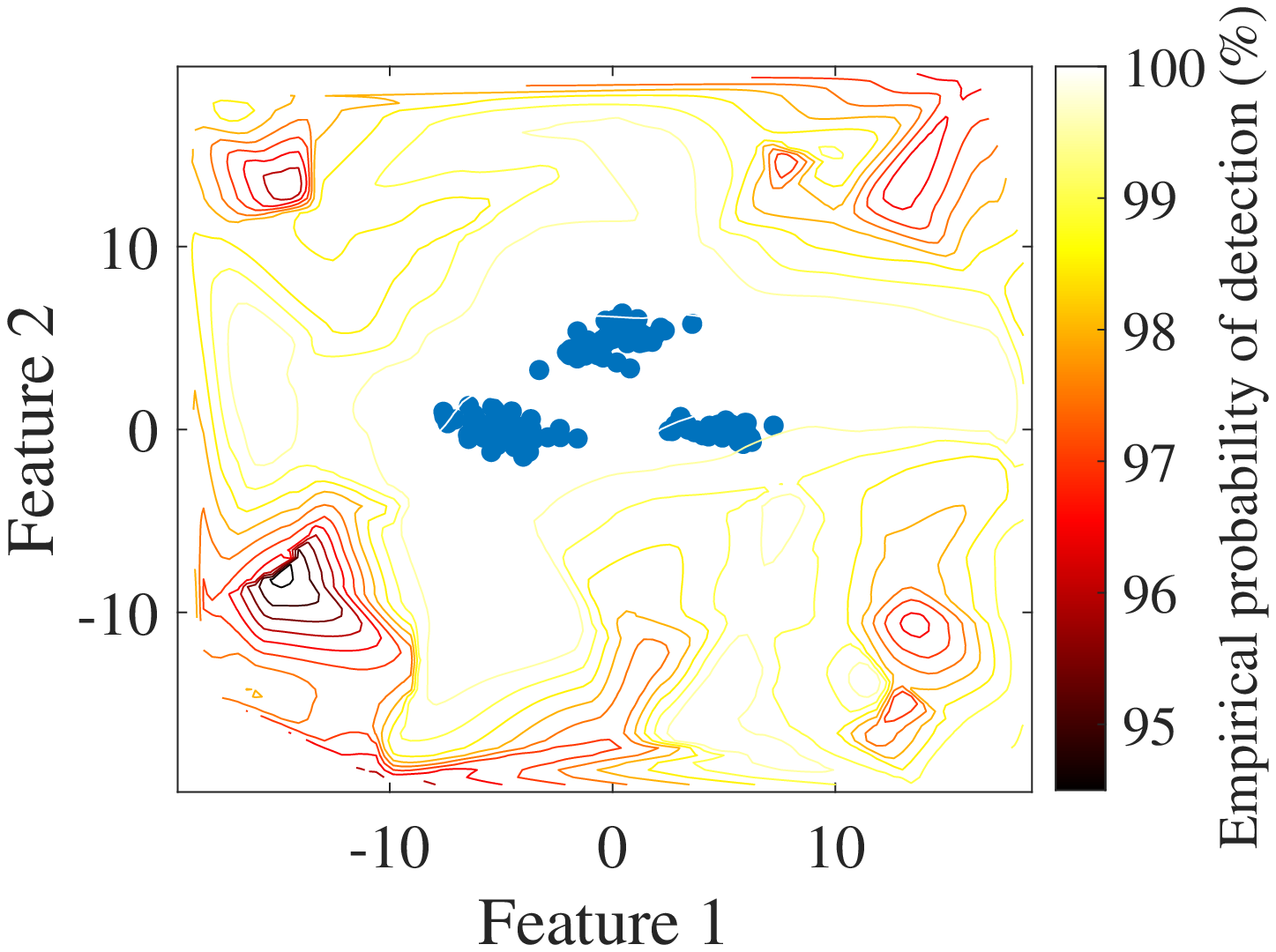}
  \caption{$\text{BIC}_{\mbox{\tiny F}t_3}$ for $N_k=50$}
 \end{subfigure}%
   \begin{subfigure}[h]{0.45\linewidth}
  \centering
  \includegraphics[width=\linewidth]{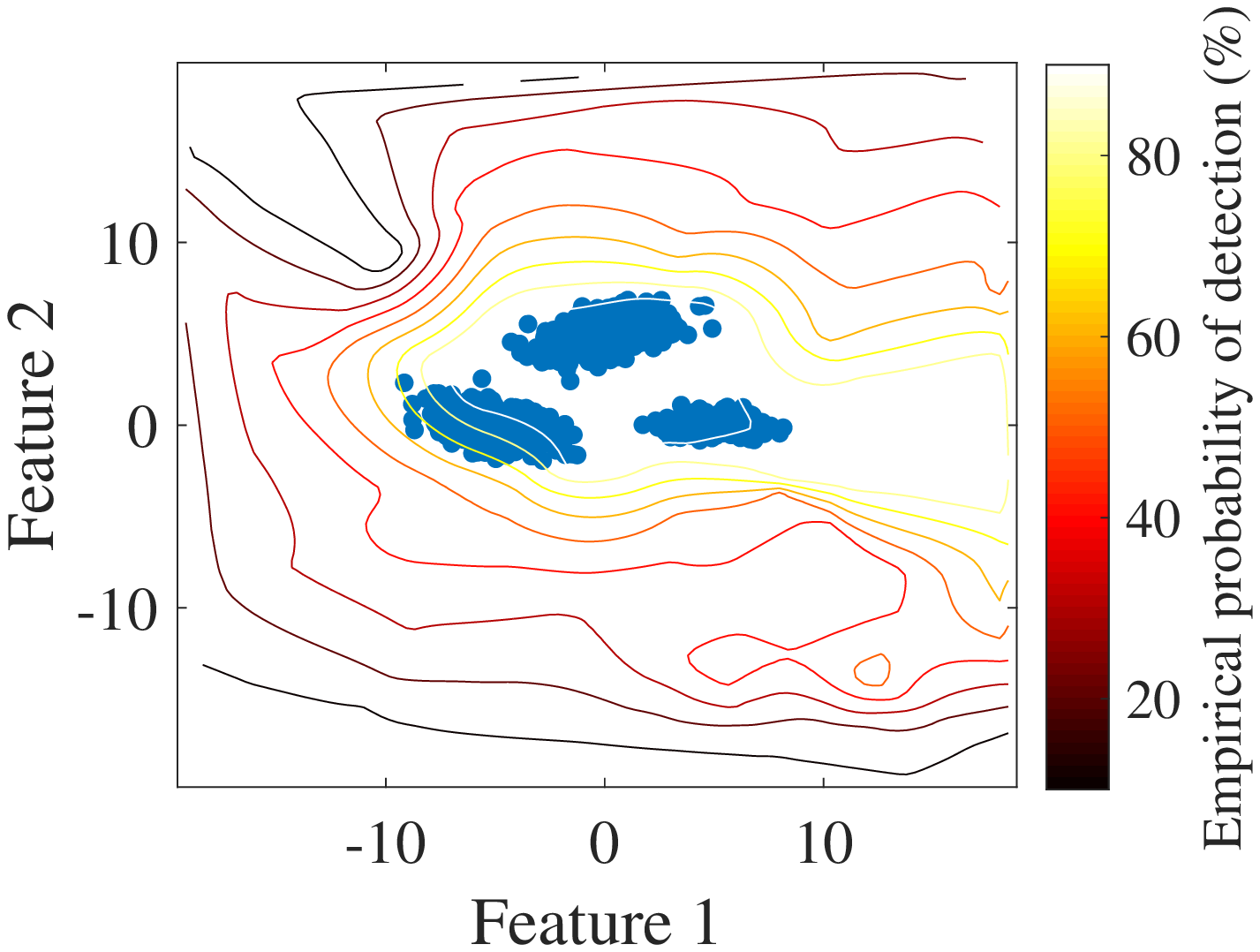}
  \caption{$\text{BIC}_{\mbox{\tiny N}}$ for $N_k=500$}
 \end{subfigure}
 \caption{Sensitivity curves of $\text{BIC}_{\mbox{\tiny F}t_3}$ and $\text{BIC}_{\mbox{\tiny N}}$ at different values of $N_k$. The sensitivity curve demonstrates the sensitivity of a method to the presence of an outlier relative to its position. }
 \label{fig:BIC_sensitivity}
\end{figure}

The second data set (Data-2), shown in Fig.~\ref{fig:Data-2-2}, contains $N_k=500$ data points in each cluster $k$ and replaces a certain percentage of the data set with outliers that are generated from a uniform distribution over the range $\left[-20, 20\right]$ on each variate. Data-2 is generated in a way that no outlier lies inside one of the data clusters. In this manner, we make sure that outliers are points that do not belong to the bulk of data. Fig.~\ref{fig:multiple_outliers} shows the empirical probability of detection as a function of the percentage of outliers $(\alpha)$. GC is able to correctly estimate the number of clusters for $\alpha>3\%$. The proposed robust criteria, $\text{BIC}_{t_3}$ and $\text{BIC}_{\mbox{\tiny F}t_3}$, and the original BIC, $\text{BIC}_{\mbox{\tiny O}t_3}$, behave similarly and are able to estimate the correct number of clusters when $\alpha\leq3\%$. The behavior of these methods is rather intuitive because, as the amount of outliers increases, then the methods try to explain the outliers by opening a new cluster. A similar trend is observed for the TBIC even though its curve decays slowly. $\text{BIC}_{\mbox{\tiny N}}$ is able to estimate the correct number of clusters $99\%$ of the time when there are no outliers in the data set. However, even $1\%$ of outliers is enough to drive $\text{BIC}_{\mbox{\tiny N}}$ into overestimating the number of clusters. 
\begin{figure}[tb]
 \centering
 \includegraphics[width=0.85\linewidth]{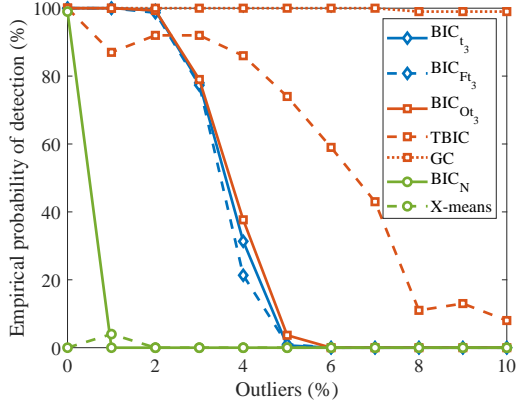}
 \caption{The empirical probability of detection in $\%$ for Data-2 as a function of the percentage of outliers.}
 \label{fig:multiple_outliers}
\end{figure}

\subsubsection{Impact of the increase in the number of features on the performance of cluster enumeration methods}
we generate realizations of the random variables $\bm{x}_k\sim t_3\left(\bm{\mu}_k,\bm{\Psi}_k\right)$, for $k=1,2$, whose cluster centroids and scatter matrices are given by $\bm{\mu}_k = c\bm{1}_{r\times 1}$ and $\bm{\Psi}_k = \bm{I}_r$, with $c\in\{0, 15\}$. For this data set, referred to as Data-3, the number of features $r$ is varied in the range $r=2,3,\ldots,55$ and the number of data points per cluster is set to $N_k=500$. Because $\nu=3$, Data-3 contains realizations of heavy tailed distributions and, as a result, the clusters contain outliers. The empirical probability of detection as a function of the number of features is displayed in Fig.~\ref{fig:dimension}. The performance of GC appears to be invariant to the increase in the number of features, while the remaining methods are affected. But, compared to the other cluster enumeration methods, GC is computationally very expensive. $\text{BIC}_{\mbox{\tiny O}t_3}$ outperforms $\text{BIC}_{t_3}$ and the TBIC when the number of features is low, while the proposed criterion $\text{BIC}_{t_3}$ outperforms both methods in high dimensions. $\text{BIC}_{\mbox{\tiny F}t_3}$ is not computed for this data set because it is computationally expensive and it is not beneficial, given the large number of samples.  
\begin{figure}[tb]
 \centering
 \includegraphics[width=0.85\linewidth]{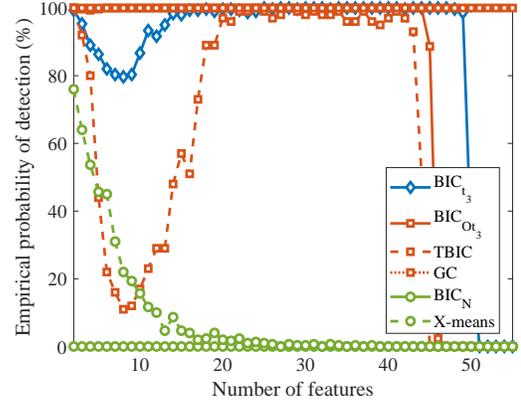}
 \caption{The empirical probability of detection in \% for Data-3 as a function of the number of features.}
 \label{fig:dimension}
\end{figure}

\subsubsection{Analysis of the sensitivity of different cluster enumeration methods to cluster overlap}
here, we use Data-2 with $1\%$ outliers and vary the distance between the second and the third centroid such that the percentage of overlap between the two clusters takes on a value from the set $\{0,5,10,25,50,75,100\}$. The empirical probability of detection as a function of the amount of overlap is depicted in Fig.~\ref{fig:overlap}. The best performance is achieved by $\text{BIC}_{t_3}$ and $\text{BIC}_{\mbox{\tiny O}t_3}$ and, remarkably, both cluster enumeration criteria are able to correctly estimate the number of clusters even when there exists $75\%$ overlap between the two clusters. As expected, when the amount of overlap is $100\%$, most methods underestimate the number of clusters to two. While it may appear that the enumeration performance of $\text{BIC}_{\mbox{\tiny N}}$ increases for increasing amounts of overlap, in fact  $\text{BIC}_{\mbox{\tiny N}}$ groups the two overlapping clusters into one and attempts to explain the outliers by opening a new cluster. A similar trend is observed for X-means. GC is inferior to the other robust methods, and experiences an increase in the empirical probability of underestimation. 
\begin{figure}[tb]
 \centering
 \includegraphics[width=0.85\linewidth]{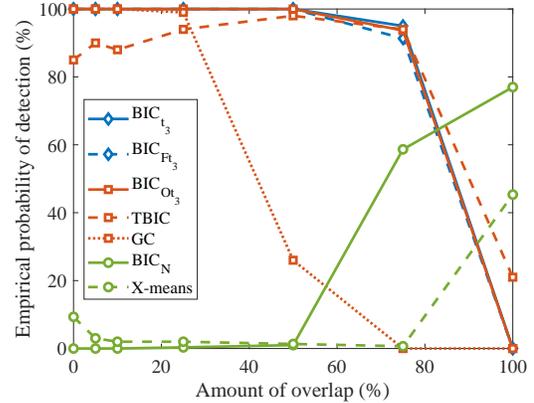}
 \caption{Impact of cluster overlap on the performance of different cluster enumeration methods.}
 \label{fig:overlap}
\end{figure}

\subsubsection{Analysis of the sensitivity of cluster enumeration methods to cluster heterogeneity}
we generate realizations of $2$-dimensional random variables $\bm{x}_k\sim t_3\left(\bm{\mu}_k,\bm{\Psi}_k\right)$, where the cluster centroids $\bm{\mu}_k$ are selected at random from a uniform distribution in the range $\left[-200,200\right]$ in each variate and the scatter matrices are set to $\bm{\Psi}_k=\bm{I}_r$ for $k=1,\ldots,5$. The data set is generated in a way that there is no overlap between the clusters. The number of data points in the first four clusters is set to $N_k=500$, while $N_5$ is allowed to take on values from the set $\{500,375,250,125,50,25,5\}$. This data set (Data-4) contains multiple outliers since each cluster contains realizations of heavy tailed $t$ distributed random variables. The empirical probability of detection as a function of the number of data points in the fifth cluster is shown in Fig.~\ref{fig:hetero}. The proposed cluster enumeration methods, $\text{BIC}_{t_3}$ and $\text{BIC}_{\mbox{\tiny F}t_3}$, are able to estimate the correct number of clusters with a high accuracy even when the fifth cluster contains only $1\%$ of the data available in the other clusters. A similar performance is observed for $\text{BIC}_{\mbox{\tiny O}t_3}$. TBIC and GC are slightly inferior in performance to the other robust cluster enumeration methods. 
When the number of data points in the fifth cluster increases, all robust methods perform well in estimating the number of clusters. Interestingly, X-means outperforms $\text{BIC}_{\mbox{\tiny N}}$ since the considered clusters are all spherical. $\text{BIC}_{\mbox{\tiny N}}$ overestimates the number of clusters and possesses the largest MAE.    
\begin{figure}[tb]
 \centering
 \includegraphics[width=0.85\linewidth]{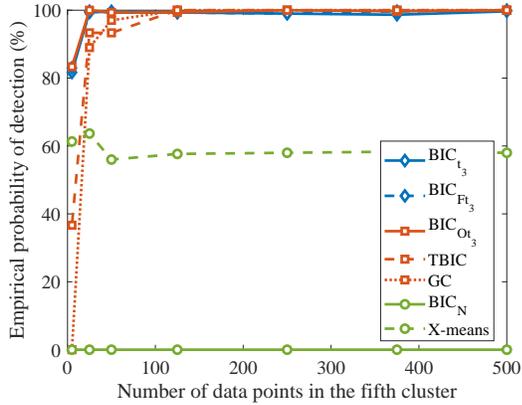}
 \caption{Impact of cluster heterogeneity on the performance of different cluster enumeration methods.}
 \label{fig:hetero}
\end{figure}

\subsection{Real Data Results}
\subsubsection{Old Faithful geyser data set}
\label{subsec:realdatares}
Old Faithful is a geyser located in Yellowstone National Park in Wyoming, United States. This data set, depicted in Fig.~\ref{fig:old_faithful}, was used in the literature for density estimation \cite{izenman2008}, time series analysis \cite{azzalini1990}, and cluster analysis \cite{bishop2006,hennig2003}. The performance of different cluster enumeration methods on the clean and contaminated versions of the Old Faithful data set is reported in Table~\ref{tab:per_real}. The contaminated version, shown in Fig.~\ref{fig:old_faithful_contaminated}, is generated by replacing a randomly selected data point with an outlier at each iteration similar to the way Data-1 was generated. Most methods are able to estimate the correct number of clusters $100\%$ of the time for the clean version of the Old Faithful data set. Our criteria,  $\text{BIC}_{t_3}$ and $\text{BIC}_{\mbox{\tiny F}t_3}$, and $\text{BIC}_{\mbox{\tiny O}t_3}$ are insensitive to the presence of a single replacement outlier, while TBIC exhibits slight sensitivity. In the presence of an outlier, the performance of $\text{BIC}_{\mbox{\tiny N}}$ deteriorates due to an increase in the empirical probability of overestimation. In fact, $\text{BIC}_{\mbox{\tiny N}}$ finds $3$ clusters $100\%$ of the time. GC shows the worst performance and possesses the highest MAE. 

Next, we replace a certain percentage of the Old Faithful data set with outliers and study the performance of different cluster enumeration methods. The outliers are generated from a uniform distribution over the range $[-20, 20]$ on each variate. The empirical probability of detection as a function of the percentage of replacement outliers is depicted in Fig.~\ref{fig:outliers_OF}. Although $\text{BIC}_{\mbox{\tiny F}t_3}$, $\text{BIC}_{t_3}$, $\text{BIC}_{\mbox{\tiny O}t_3}$, and TBIC are able to estimate the correct number of clusters reasonably well for clean data, their performance deteriorates quickly as the percentage of outliers increases. $\text{BIC}_{\mbox{\tiny N}}$, X-means, and GC overestimate the number of clusters for $100\%$ of the cases. 
\begin{figure}[tb]
 \centering
 \begin{subfigure}[tb]{0.45\linewidth}
  \centering
  \includegraphics[width=\linewidth]{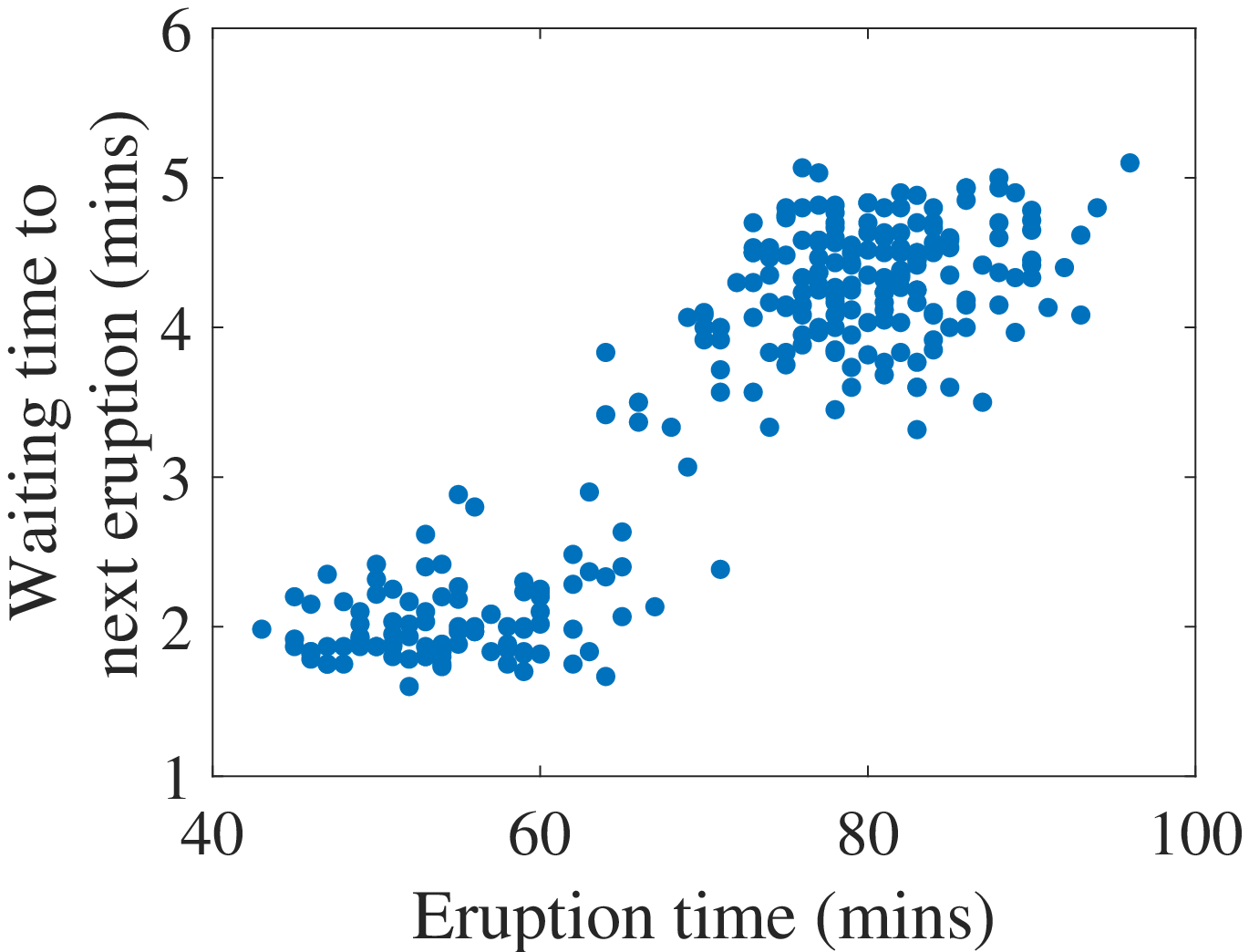}
  \caption{Clean data}
  \label{fig:old_faithful}
 \end{subfigure}%
 \begin{subfigure}[tb]{0.45\linewidth}
  \centering
  \includegraphics[width=\linewidth]{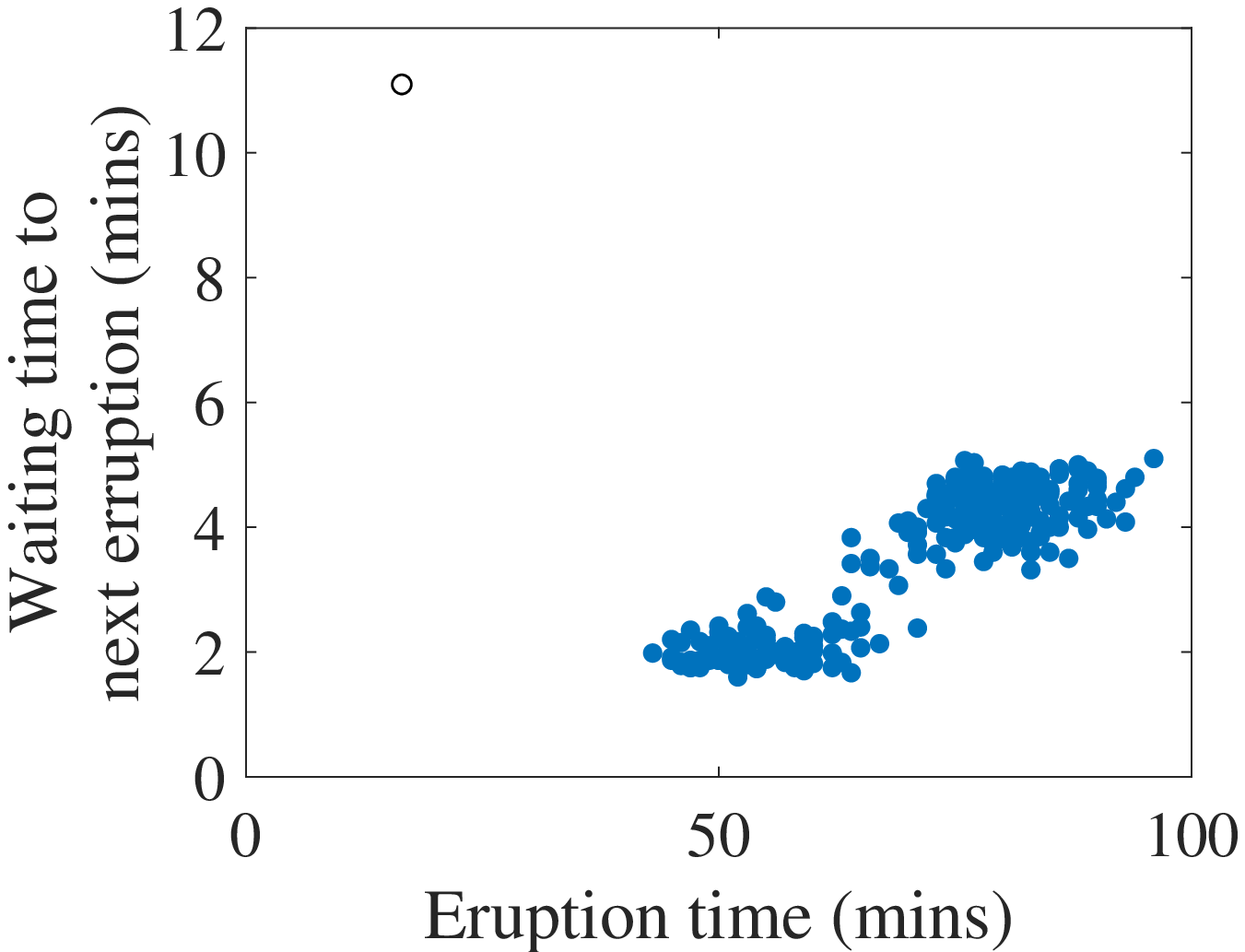}
  \caption{Contaminated data}
  \label{fig:old_faithful_contaminated}
 \end{subfigure}
 \caption{Clean and contaminated versions of the Old Faithful geyser data set.}
\end{figure}

\begin{table}[tb]
 \centering
 \caption{The performance of different cluster enumeration methods on a clean and a contaminated version of the Old Faithful data set.}
 \begin{tabular}{cccc}
  \toprule
    & & Old Faithful & \makecell{Old Faithful with\\a single outlier} \\
  \midrule
  \multirow{2}{*}{$\text{BIC}_{t_3}$} &$p_{\text{det}}$ & $\bm{100}$ & $\bm{100}$  \\
  & MAE & $\bm{0}$ & $\bm{0}$ \\
  \cmidrule{2-4}
  \multirow{2}{*}{$\text{BIC}_{\mbox{\tiny F}t_3}$} &$p_{\text{det}}$ & $\bm{100}$ & $\bm{100}$  \\
  & MAE & $\bm{0}$ & $\bm{0}$  \\
  \midrule
  \multirow{2}{*}{$\text{BIC}_{\mbox{\tiny O}t_3}$ \cite{andrews2012,mcnicholas2012}} & $p_{\text{det}}$  & $\bm{100}$ & $\bm{100}$ \\
  & MAE & $\bm{0}$ & $\bm{0}$ \\
  \cmidrule{2-4}
  \multirow{2}{*}{$\text{TBIC}$ \cite{neykov2007}} & $p_{\text{det}}$  & $\bm{100}$ & $92.03$    \\
  & MAE & $\bm{0}$ & $0.09$ \\
  \cmidrule{2-4}
  \multirow{2}{*}{GC \cite{binder2017}} & $p_{\text{det}}$ & $0$ & $0$ \\
  & MAE & $10.34$ & $10.26$ \\
  \midrule
  \multirow{2}{*}{$\text{BIC}_{\mbox{\tiny N}}$ \cite{teklehaymanot22017}} & $p_{\text{det}}$  & $\bm{100}$ & $5.08$ \\
  & MAE & $\bm{0}$ & $1.36$ \\
  \cmidrule{2-4}
  \multirow{2}{*}{X-means \cite{pelleg2000}} & $p_{\text{det}}$  & $0$ & $0$\\
  & MAE & $2$ & $2$ \\
  \bottomrule
 \end{tabular}
\label{tab:per_real}
\end{table}
\begin{figure}[tb]
 \centering
 \includegraphics[width=0.85\linewidth]{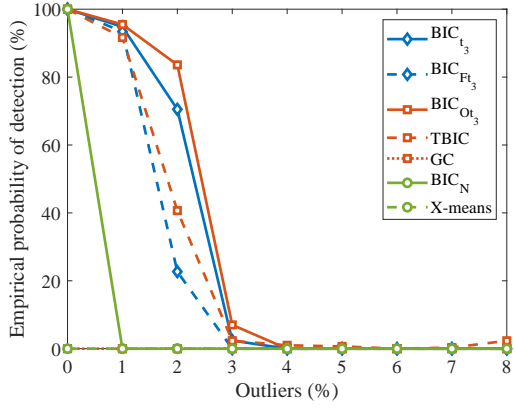}
 \caption{Empirical probability of detection in \% for the Old Faithful data set as a function of the percentage of replacement outliers.}
 \label{fig:outliers_OF}
\end{figure}

\section{Conclusion}
\label{sec:conclusion}
We derived a robust cluster enumeration criterion by formulating the problem of estimating the number of clusters as maximization of the posterior probability of multivariate $t_\nu$ candidate models. The derivation is based on Bayes' theorem and asymptotic approximations. Further, we refined the penalty term of the robust criterion for the finite sample regime. Since both robust criteria require cluster parameter estimates as an input, we proposed a two-step cluster enumeration algorithm that uses the EM algorithm to partition the data and estimate cluster parameters prior to the calculation of either of the robust criteria. The following two statements can be made with respect to the original BIC: First, the asymptotic criterion derived specifically for cluster analysis has a different penalty term compared to the original BIC based on multivariate $t_\nu$ candidate models. Second, since the derived asymptotic criterion converges to the original BIC as data size goes to infinity, we are able to provide a justification for the use of the original BIC with multivariate $t_\nu$ candidate models. The performance of the proposed cluster enumeration algorithm is demonstrated using numerical and real data experiments. We showed superiority of the proposed robust cluster enumeration methods in estimating the number of clusters in contaminated data sets.

\appendices

\section{Maximum Likelihood Estimators of the Parameters of the Multivariate $t_\nu$ Distribution}
\label{app:A}
The log-likelihood function of the data points that belong to the $m$th cluster is given by 
\begin{align}
  \log \mathcal{L}(\bm{\theta}_m|\mathcal{X}_m) &\!=\! \log \prod_{\bm{x}_n\in\mathcal{X}_m} p(\bm{x}_n\in\mathcal{X}_m) f(\bm{x}_n|\bm{\theta}_m) \nonumber \\
  &\!=N_m\log\frac{N_m}{N} + N_m\log\frac{\Gamma\left((\nu_m+r)/2\right)}{\Gamma\left(\nu_m/2\right)(\pi\nu_m)^{r/2}} \nonumber \\
  &-\!\frac{N_m}{2}\log|\bm{\Psi}_m|\!-\! \frac{(\nu_m+r)}{2}\!\!\!\!\sum_{\bm{x}_n\in\mathcal{X}_m}\!\!\!\!\log\!\left(\!1+\frac{\delta_n}{\nu_m}\!\right)\!, 
 \label{eq:loglikeli}
\end{align}
where $\delta_n=(\bm{x}_n-\bm{\mu}_m)^\top\bm{\Psi}_m^{-1}(\bm{x}_n-\bm{\mu}_m)$ is the squared Mahalanobis distance and $\Gamma(\cdot)$ is the gamma function. To find the maximum likelihood estimators of the centroid $\bm{\mu}_m$ and the scatter matrix $\bm{\Psi}_m$, we first derivate the log-likelihood function with respect to each parameter, which results in 
\begin{align}
\frac{\partial\log \mathcal{L}(\bm{\theta}_m|\mathcal{X}_m)}{\partial\bm{\mu}_m}  
&= \frac{1}{2}\sum_{\bm{x}_n\in\mathcal{X}_m}w_n\biggl(\frac{d\bm{\mu}_m^\top}{d\bm{\mu}_m}\bm{\Psi}_m^{-1}\tilde{\bm{x}}_n \nonumber \\
&+ \tilde{\bm{x}}_n^\top\bm{\Psi}_m^{-1}\frac{d\bm{\mu}_m}{d\bm{\mu}_m}\biggr) \nonumber \\
&= \sum_{\bm{x}_n\in\mathcal{X}_m}w_n\tilde{\bm{x}}_n^\top\bm{\Psi}_m^{-1}
\label{eq:firstdermu} \\
\frac{\partial\log \mathcal{L}(\bm{\theta}_m|\mathcal{X}_m)}{\partial\bm{\Psi}_m} 
&=-\frac{N_m}{2}\tr\left(\bm{\Psi}_m^{-1}\frac{d\bm{\Psi}_m}{d\bm{\Psi}_m}\right) \nonumber \\
&- \frac{1}{2}\sum_{\bm{x}_n\in\mathcal{X}_m}\frac{\nu_m+r}{\nu_m+\delta_n}\frac{\partial\delta_n}{\partial\bm{\Psi}_m} \nonumber \\
&=-\frac{N_m}{2}\tr\left(\bm{\Psi}_m^{-1}\frac{d\bm{\Psi}_m}{d\bm{\Psi}_m}\right) \nonumber \\
&+\frac{1}{2}\sum_{\bm{x}_n\in\mathcal{X}_m}w_n\tilde{\bm{x}}_n^\top\bm{\Psi}_m^{-1}\frac{d\bm{\Psi}_m}{d\bm{\Psi}_m}\bm{\Psi}_m^{-1}\tilde{\bm{x}}_n,
\label{eq:firdersigma}
\end{align}
where $\tilde{\bm{x}}_n=\bm{x}_n-\bm{\mu}_m$ and
\begin{equation}
w_n = \frac{\nu_m+r}{\nu_m+\delta_n}
\label{eq:weight}
\end{equation}
is the weight given to $\bm{x}_n$. Then, setting Eqs.~\eqref{eq:firstdermu} and \eqref{eq:firdersigma} to zero and simplifying the resulting expressions result in
\begin{align}
\hat{\bm{\mu}}_m &= \frac{\sum_{\bm{x}_n\in\mathcal{X}_m}w_n\bm{x}_n}{\sum_{\bm{x}_n\in\mathcal{X}_m}w_n} \label{eq:muhat_2} \\
\hat{\bm{\Psi}}_m &= \frac{1}{N_m}\sum_{\bm{x}_n\in\mathcal{X}_m}w_n\tilde{\bm{x}}_n\tilde{\bm{x}}_n^\top. \label{eq:psihat}
\end{align}

Note that, to make the paper self contained, the vector and matrix differentiation rules used in Appendices~\ref{app:A} and \ref{app:B} are discussed in Appendix~\ref{app:d}. 

\section{Proof of Theorem 1}
\label{app:B}
Proving Theorem 1 requires finding an asymptotic approximation for $|\hat{\bm{J}}_m|$ in Eq.~\eqref{eq:finalposterior} and, consequently, deriving an expression for $\text{BIC}_{t_\nu}(M_l)$. We start the proof by writing the Fisher information matrix in a compact form as
\begin{equation}
\hat{\bm{J}}_m \!\!=\!\! 
\begin{bmatrix}
-\frac{\partial^2\log \mathcal{L}(\bm{\theta}_m|\mathcal{X}_m)}{\partial\bm{\mu}_m\partial\bm{\mu}_m^\top}\big|_{\bm{\theta}_m=\hat{\bm{\theta}}_m} & -\frac{\partial^2\log \mathcal{L}(\bm{\theta}_m|\mathcal{X}_m)}{\partial\bm{\mu}_m\partial\bm{\Psi}_m}\big|_{\bm{\theta}_m=\hat{\bm{\theta}}_m} \\
-\frac{\partial^2\log \mathcal{L}(\bm{\theta}_m|\mathcal{X}_m)}{\partial\bm{\Psi}_m\partial\bm{\mu}_m^\top}\big|_{\bm{\theta}_m=\hat{\bm{\theta}}_m} & -\frac{\partial^2\log \mathcal{L}(\bm{\theta}_m|\mathcal{X}_m)}{\partial\bm{\Psi}_m\partial\bm{\Psi}_m}\big|_{\bm{\theta}_m=\hat{\bm{\theta}}_m}
\end{bmatrix}\!\!. 
\label{eq:FIM1}
\end{equation} 
To simplify notation Eq.~\eqref{eq:FIM1} is written as
\begin{equation}
\hat{\bm{J}}_m =
\begin{bmatrix}
-\hat{\bm{J}}_{\bm{\mu}\bm{\mu}^\top} & -\hat{\bm{J}}_{\bm{\mu}\bm{\Psi}} \\
-\hat{\bm{J}}_{\bm{\Psi}\bm{\mu}^\top} & -\hat{\bm{J}}_{\bm{\Psi}\bm{\Psi}}
\end{bmatrix}.
\label{eq:FIM2}
\end{equation}

The first diagonal element of the Fisher information matrix is the derivative of Eq.~\eqref{eq:firstdermu} with respect to $\bm{\mu}_m^\top$ which is given by
\begin{align}
\bm{J}_{\bm{\mu}\bm{\mu}^\top}  
&= \sum_{\bm{x}_n\in\mathcal{X}_m}\biggl(\frac{2w_n^2}{\nu_m+r}\tilde{\bm{x}}_n^\top\bm{\Psi}_m^{-1}\frac{d\bm{\mu}_m}{d\bm{\mu}_m^\top}\tilde{\bm{x}}_n^\top\bm{\Psi}_m^{-1} \nonumber \\
&- w_n\frac{d\bm{\mu}_m^\top}{d\bm{\mu}_m^\top}\bm{\Psi}_m^{-1} \biggr) 
\label{eq:secdermu1}
\end{align}
Applying the vec operator, Eq.~\eqref{eq:secdermu1} is further simplified to
\begin{align}
\bm{J}_{\bm{\mu}\bm{\mu}^\top}  &=  \sum_{\bm{x}_n\in\mathcal{X}_m}\biggl(\frac{2w_n^2}{\nu_m+r}\text{vec}\left(\tilde{\bm{x}}_n^\top\bm{\Psi}_m^{-1}\frac{d\bm{\mu}_m}{d\bm{\mu}_m^\top}\tilde{\bm{x}}_n^\top\bm{\Psi}_m^{-1}\right) \nonumber \\
&- w_n\text{vec}\left(\frac{d\bm{\mu}_m^\top}{d\bm{\mu}_m^\top}\bm{\Psi}_m^{-1}\right) \biggr) \nonumber \\
&= \sum_{\bm{x}_n\in\mathcal{X}_m}\biggl(\frac{2w_n^2}{\nu_m+r}(\bm{\Psi}_m^{-1}\tilde{\bm{x}}_n\otimes\tilde{\bm{x}}_n^\top\bm{\Psi}_m^{-1})\text{vec}\left(\frac{d\bm{\mu}_m}{d\bm{\mu}_m^\top}\right) \nonumber \\
&- w_n(\bm{\Psi}_m^{-1}\otimes\bm{I}_1)\text{vec}\left(\frac{d\bm{\mu}_m^\top}{d\bm{\mu}_m^\top}\right)\biggr) \nonumber \\
&= \frac{2}{\nu_m+r}\bm{\Psi}_m^{-1}\left(\sum_{\bm{x}_n\in\mathcal{X}_m}w_n^2\tilde{\bm{x}}_n\tilde{\bm{x}}_n^\top\right)\bm{\Psi}_m^{-1} \nonumber \\
&- \bm{\Psi}_m^{-1}\sum_{\bm{x}_n\in\mathcal{X}_m}w_n.
\label{eq:secdermu2}
\end{align}
Evaluating Eq.~\eqref{eq:secdermu2} at the maximum likelihood estimates result in
\begin{align}
\hat{\bm{J}}_{\bm{\mu}\bm{\mu}^\top} &= \frac{2}{\nu_m+r}\hat{\bm{\Psi}}_m^{-1}\left(\sum_{\bm{x}_n\in\mathcal{X}_m}w_n^2\tilde{\bm{x}}_n\tilde{\bm{x}}_n^\top\right)\hat{\bm{\Psi}}_m^{-1} \nonumber \\
&- \hat{\bm{\Psi}}_m^{-1}\sum_{\bm{x}_n\in\mathcal{X}_m}w_n,
\label{eq:secdermu3}
\end{align}
where, in this case, $\tilde{\bm{x}}_n = \bm{x}_n-\hat{\bm{\mu}}_m$.

Next, we move to the off-diagonal elements of the Fisher information matrix. Due to the symmetry of $\hat{\bm{J}}_m$,  the following holds:
\begin{equation}
\hat{\bm{J}}_{\bm{\Psi}\bm{\mu}^\top}=\hat{\bm{J}}_{\bm{\mu}\bm{\Psi}}^\top.
\end{equation}
As a result, in this manuscript, we show only the derivation of $\hat{\bm{J}}_{\bm{\mu}\bm{\Psi}}$. To this end, derivating Eq.~\eqref{eq:firstdermu} with respect to $\bm{\Psi}_m$ results in 
\begin{align}
\bm{J}_{\bm{\mu}\bm{\Psi}} 
&= -\sum_{\bm{x}_n\in\mathcal{X}_m}\frac{(\nu_m+r)}{(\nu_m+\delta_n)^2}\frac{\partial \delta_n}{\partial\bm{\Psi}_m}\tilde{\bm{x}}_n^\top\bm{\Psi}_m^{-1} \nonumber \\
&+ \sum_{\bm{x}_n\in\mathcal{X}_m}w_n\tilde{\bm{x}}_n^\top\frac{d\bm{\Psi}_m^{-1}}{d\bm{\Psi}_m} \nonumber \\
&= \frac{1}{\nu_m+r}\sum_{\bm{x}_n\in\mathcal{X}_m}w_n^2\tilde{\bm{x}}_n^\top\bm{\Psi}_m^{-1}\frac{d\bm{\Psi}_m}{d\bm{\Psi}_m}\bm{\Psi}_m^{-1}\tilde{\bm{x}}_n\tilde{\bm{x}}_n^\top\bm{\Psi}_m^{-1} \nonumber \\
&- \sum_{\bm{x}_n\in\mathcal{X}_m}w_n\tilde{\bm{x}}_n^\top\bm{\Psi}_m^{-1}\frac{d\bm{\Psi}_m}{d\bm{\Psi}_m}\bm{\Psi}_m^{-1}. 
\label{eq:secdermusigma1}
\end{align}
To further simplify Eq.~\eqref{eq:secdermusigma1} the vec operator is applied to it and this results in   
\begin{align}
\bm{J}_{\bm{\mu}\bm{\Psi}}  &= \frac{1}{\nu_m+r}\!\!\sum_{\bm{x}_n\in\mathcal{X}_m}\!\!\!\!w_n^2\text{vec}\left(\tilde{\bm{x}}_n^\top\bm{\Psi}_m^{-1}\frac{d\bm{\Psi}_m}{d\bm{\Psi}_m}\bm{\Psi}_m^{-1}\tilde{\bm{x}}_n\tilde{\bm{x}}_n^\top\bm{\Psi}_m^{-1}\right) \nonumber \\
&- \sum_{\bm{x}_n\in\mathcal{X}_m}\!\!\!\!w_n\text{vec}\left(\tilde{\bm{x}}_n^\top\bm{\Psi}_m^{-1}\frac{d\bm{\Psi}_m}{d\bm{\Psi}_m}\bm{\Psi}_m^{-1}\right) \nonumber \\
&= \frac{1}{\nu_m+r}\sum_{\bm{x}_n\in\mathcal{X}_m}\!\!\!\!\biggl(w_n^2\left(\bm{\Psi}_m^{-1}\tilde{\bm{x}}_n\tilde{\bm{x}}_n^\top\bm{\Psi}_m^{-1}\otimes\tilde{\bm{x}}_n^\top\bm{\Psi}_m^{-1}\right)\nonumber \\
&*\frac{d\text{vec}(\bm{\Psi}_m)}{d\bm{\Psi}_m}\biggr) - \!\!\!\sum_{\bm{x}_n\in\mathcal{X}_m}\!\!\!\!w_n\left(\bm{\Psi}_m^{-1}\otimes\tilde{\bm{x}}_n^\top\bm{\Psi}_m^{-1}\right)\frac{d\text{vec}(\bm{\Psi}_m)}{d\bm{\Psi}_m}. 
\label{eq:secdermusigma2}
\end{align}
The scatter matrix $\bm{\Psi}_m$, $m=1,\ldots,l$, is a symmetric and positive definite matrix. Hence, $\text{vec}(\bm{\Psi}_m)=\bm{D}\bm{u}_m$, where $\text{vec}(\bm{\Psi}_m)\in\mathbb{R}^{r^2\times 1}$ represents the stacking of the elements of $\bm{\Psi}_m$ into a long column vector, $\bm{D}\in\mathbb{R}^{r^2\times\frac{1}{2}r(r+1)}$ denotes the duplication matrix, and $\bm{u}_m\in\mathbb{R}^{\frac{1}{2}r(r+1)\times 1}$ contains the unique elements of $\bm{\Psi}_m$ \cite{magnus2007}. Taking the symmetry of the scatter matrix into account and replacing $d\bm{\Psi}_m$ by $d\bm{u}_m$, Eq.~\eqref{eq:secdermusigma2} simplify to 
\begin{align}
\bm{J}_{\bm{\mu}\bm{\Psi}} 
&= \frac{1}{\nu_m+r}\sum_{\bm{x}_n\in\mathcal{X}_m}w_n^2\left(\bm{\Psi}_m^{-1}\tilde{\bm{x}}_n\tilde{\bm{x}}_n^\top\bm{\Psi}_m^{-1}\otimes\tilde{\bm{x}}_n^\top\bm{\Psi}_m^{-1}\right)\bm{D} \nonumber \\
&- \left(\bm{\Psi}_m^{-1}\otimes\left(\sum_{\bm{x}_n\in\mathcal{X}_m}w_n\tilde{\bm{x}}_n^\top\right)\bm{\Psi}_m^{-1}\right)\bm{D}.
\label{eq:secdermusigma3}
\end{align}
Evaluating Eq.~\eqref{eq:secdermusigma3} at the maximum likelihood estimates results in 
\begin{align}
\hat{\bm{J}}_{\bm{\mu}\bm{\Psi}}&= \frac{1}{\nu_m+r}\sum_{\bm{x}_n\in\mathcal{X}_m}w_n^2\left(\hat{\bm{\Psi}}_m^{-1}\tilde{\bm{x}}_n\tilde{\bm{x}}_n^\top\hat{\bm{\Psi}}_m^{-1}\otimes\tilde{\bm{x}}_n^\top\hat{\bm{\Psi}}_m^{-1}\right)\bm{D} 
\label{eq:secdermusigma4}
\end{align}
since
\begin{equation*}
\sum_{\bm{x}_n\in\mathcal{X}_m}w_n\left(\bm{x}_n-\hat{\bm{\mu}}_m\right) = \!\!\sum_{\bm{x}_n\in\mathcal{X}_m}w_n\bm{x}_n - \sum_{\bm{x}_n\in\mathcal{X}_m}w_n\hat{\bm{\mu}}_m = 0. 
\end{equation*}

The last element of the Fisher information matrix in our derivation is $\hat{\bm{J}}_{\bm{\Psi}\bm{\Psi}}$. We proceed with the derivation of $\hat{\bm{J}}_{\bm{\Psi}\bm{\Psi}}$ by applying the vec operator to Eq.~\eqref{eq:firdersigma} and simplifying it as follows:
\begin{align}
\frac{\partial\text{vec}\log \mathcal{L}(\bm{\theta}_m|\mathcal{X}_m)}{\partial\bm{\Psi}_m} &= -\frac{N_m}{2}\text{vec}\left(\tr\left(\bm{\Psi}_m^{-1}\frac{d\bm{\Psi}_m}{d\bm{\Psi}_m}\right)\right) \nonumber \\
&+\frac{1}{2}\!\!\!\sum_{\bm{x}_n\in\mathcal{X}_m}\!\!\!w_n\text{vec}\left(\!\!\tilde{\bm{x}}_n^\top\bm{\Psi}_m^{-1}\frac{d\bm{\Psi}_m}{d\bm{\Psi}_m}\bm{\Psi}_m^{-1}\tilde{\bm{x}}_n\!\!\right) \nonumber \\
&= -\frac{N_m}{2}\text{vec}\left(\bm{\Psi}_m^{-1}\right)^\top\text{vec}\left(\frac{d\bm{\Psi}_m}{d\bm{\Psi}_m}\right) \nonumber \\
&+ \frac{1}{2}\sum_{\bm{x}_n\in\mathcal{X}_m}\biggl(w_n\left(\tilde{\bm{x}}_n^\top\bm{\Psi}_m^{-1}\otimes\tilde{\bm{x}}_n^\top\bm{\Psi}_m^{-1}\right)\nonumber \\
&*\text{vec}\left(\frac{d\bm{\Psi}_m}{d\bm{\Psi}_m}\right)\biggr) \nonumber \\
&= -\frac{N_m}{2}\text{vec}\left(\bm{\Psi}_m^{-1}\right)^\top\bm{D} \nonumber \\
&+ \frac{1}{2}\!\!\sum_{\bm{x}_n\in\mathcal{X}_m}\!\!\!w_n\left(\tilde{\bm{x}}_n^\top\bm{\Psi}_m^{-1}\otimes\tilde{\bm{x}}_n^\top\bm{\Psi}_m^{-1}\right)\bm{D}
\label{eq:firdersigma2}
\end{align}
Then, to obtain a final expression for $\hat{\bm{J}}_{\bm{\Psi}\bm{\Psi}}$ we derivate Eq.~\eqref{eq:firdersigma2} by $\bm{\Psi}_m$ which results in 
\begin{align}
\bm{J}_{\bm{\Psi}\bm{\Psi}} 
&= \frac{N_m}{2}\text{vec}\left(\bm{\Psi}_m^{-1}\frac{d\bm{\Psi}_m}{d\bm{\Psi}_m}\bm{\Psi}_m^{-1}\right)^\top\bm{D} \nonumber \\
&+ \frac{1}{2}\sum_{\bm{x}_n\in\mathcal{X}_m}\frac{\partial w_n}{\partial\bm{\Psi}_m}\left(\tilde{\bm{x}}_n^\top\bm{\Psi}_m^{-1}\otimes\tilde{\bm{x}}_n^\top\bm{\Psi}_m^{-1}\right)\bm{D} \nonumber \\
&+ \frac{1}{2}\sum_{\bm{x}_n\in\mathcal{X}_m}w_n\frac{\partial}{\partial\bm{\Psi}_m}\left(\tilde{\bm{x}}_n^\top\bm{\Psi}_m^{-1}\otimes\tilde{\bm{x}}_n^\top\bm{\Psi}_m^{-1}\right)\bm{D}. 
\label{eq:secdersigma1}
\end{align}
By applying the vec operator, Eq.~\eqref{eq:secdersigma1} is further simplified as
\begin{align}
\bm{J}_{\bm{\Psi}\bm{\Psi}} 
&= \frac{N_m}{2}\bm{D}^\top\text{vec}\left(\bm{\Psi}_m^{-1}\frac{d\bm{\Psi}_m}{d\bm{\Psi}_m}\bm{\Psi}_m^{-1}\right) \nonumber \\
&+ \frac{1}{2}\sum_{\bm{x}_n\in\mathcal{X}_m}\frac{w_n^2}{\nu_m+r}\text{vec}\biggl(\tilde{\bm{x}}_n^\top\bm{\Psi}_m^{-1}\frac{d\bm{\Psi}_m}{d\bm{\Psi}_m}\bm{\Psi}_m^{-1}\tilde{\bm{x}}_n\nonumber \\
&*\left(\tilde{\bm{x}}_n^\top\bm{\Psi}_m^{-1}\otimes\tilde{\bm{x}}_n^\top\bm{\Psi}_m^{-1}\right)\bm{D}\biggr) \nonumber \\
&+ \frac{1}{2}\sum_{\bm{x}_n\in\mathcal{X}_m}w_n\bm{D}^\top\text{vec}\left(\frac{\partial}{\partial\bm{\Psi}_m}\left(\tilde{\bm{x}}_n^\top\bm{\Psi}_m^{-1}\otimes\tilde{\bm{x}}_n^\top\bm{\Psi}_m^{-1}\right)\right) \nonumber \\
&= \frac{N_m}{2}\bm{D}^\top(\bm{\Psi}_m^{-1}\otimes\bm{\Psi}_m^{-1})\frac{d\text{vec}\left(\bm{\Psi}_m\right)}{d\bm{\Psi}_m} \nonumber \\
&+ \frac{1}{2}\!\!\sum_{\bm{x}_n\in\mathcal{X}_m}\!\!\!\frac{w_n^2}{\nu_m+r}\biggl(\left(\bm{\Psi}_m^{-1}\tilde{\bm{x}}_n\left(\tilde{\bm{x}}_n^\top\bm{\Psi}_m^{-1}\otimes\tilde{\bm{x}}_n^\top\bm{\Psi}_m^{-1}\right)\bm{D}\right)^\top \nonumber \\
&\otimes\tilde{\bm{x}}_n^\top\bm{\Psi}_m^{-1}\biggr)\frac{d\text{vec}\left(\bm{\Psi}_m\right)}{d\bm{\Psi}_m} \nonumber \\
&+ \frac{1}{2}\sum_{\bm{x}_n\in\mathcal{X}_m}\biggl(w_n\bm{D}^\top\left(\bm{I}_r\otimes\bm{K}_{r,1}\otimes\bm{I}_1\right) \nonumber \\
&*\biggl[\left(\bm{I}_r\otimes\text{vec}\left(\tilde{\bm{x}}_n^\top\bm{\Psi}_m^{-1}\right)\right)\frac{\partial}{\partial\bm{\Psi}_m}\left(\text{vec}\left(\tilde{\bm{x}}_n^\top\bm{\Psi}_m^{-1}\right)\right) \nonumber \\
&+\left(\text{vec}\left(\tilde{\bm{x}}_n^\top\bm{\Psi}_m^{-1}\right)\otimes\bm{I}_r\right)\frac{\partial}{\partial\bm{\Psi}_m}\left(\text{vec}\left(\tilde{\bm{x}}_n^\top\bm{\Psi}_m^{-1}\right)\right)\biggr]\biggr) \nonumber \\
&= \frac{N_m}{2}\bm{D}^\top(\bm{\Psi}_m^{-1}\otimes\bm{\Psi}_m^{-1})\bm{D} \nonumber \\
&+ \frac{1}{2}\sum_{\bm{x}_n\in\mathcal{X}_m}\frac{w_n^2}{\nu_m+r}\bm{D}^\top\left(\bm{\Psi}_m^{-1}\tilde{\bm{x}}_n\otimes\bm{\Psi}_m^{-1}\tilde{\bm{x}}_n\right)\nonumber \\
&*\left(\tilde{\bm{x}}_n^\top\bm{\Psi}_m^{-1}\otimes\tilde{\bm{x}}_n^\top\bm{\Psi}_m^{-1}\right)\bm{D}\nonumber \\
&- \frac{1}{2}\sum_{\bm{x}_n\in\mathcal{X}_m}\biggl(w_n\bm{D}^\top\bm{I}_{r^2}\biggl[\left(\bm{I}_r\otimes\bm{\Psi}_m^{-1}\tilde{\bm{x}}_n\right) \nonumber \\
&+ \left(\bm{\Psi}_m^{-1}\tilde{\bm{x}}_n\otimes\bm{I}_r\right)\biggr]\text{vec}\left(\tilde{\bm{x}}_n^\top\bm{\Psi}_m^{-1} \frac{d\bm{\Psi}_m}{d\bm{\Psi}_m}\bm{\Psi}_m^{-1}\right) \nonumber \\
&= \frac{N_m}{2}\bm{D}^\top(\bm{\Psi}_m^{-1}\otimes\bm{\Psi}_m^{-1})\bm{D} \nonumber \\
&+ \frac{1}{2}\sum_{\bm{x}_n\in\mathcal{X}_m}\frac{w_n^2}{\nu_m+r}\bm{D}^\top\biggl(\bm{\Psi}_m^{-1}\tilde{\bm{x}}_n\tilde{\bm{x}}_n\bm{\Psi}_m^{-1}\nonumber \\
&\otimes\bm{\Psi}_m^{-1}\tilde{\bm{x}}_n\tilde{\bm{x}}_n\bm{\Psi}_m^{-1}\biggr)\bm{D}\nonumber \\
&- \frac{1}{2}\sum_{\bm{x}_n\in\mathcal{X}_m}\biggl(w_n\bm{D}^\top\biggl[\left(\bm{I}_r\otimes\bm{\Psi}_m^{-1}\tilde{\bm{x}}_n\right) \nonumber \\
&*\left(\bm{\Psi}_m^{-1}\otimes\tilde{\bm{x}}_n^\top\bm{\Psi}_m^{-1}\right) + \left(\bm{\Psi}_m^{-1}\tilde{\bm{x}}_n\otimes\bm{I}_r\right) \nonumber \\
&*\left(\bm{\Psi}_m^{-1}\otimes\tilde{\bm{x}}_n^\top\bm{\Psi}_m^{-1}\right)\biggr]\bm{D}\biggr) \nonumber \\
&= \frac{N_m}{2}\bm{D}^\top(\bm{\Psi}_m^{-1}\otimes\bm{\Psi}_m^{-1})\bm{D} \nonumber \\
&+ \frac{1}{2(\nu_m+r)}\sum_{\bm{x}_n\in\mathcal{X}_m}w_n^2\bm{D}^\top\biggl(\bm{\Psi}_m^{-1}\tilde{\bm{x}}_n\tilde{\bm{x}}_n\bm{\Psi}_m^{-1}\nonumber \\
&\otimes\bm{\Psi}_m^{-1}\tilde{\bm{x}}_n\tilde{\bm{x}}_n\bm{\Psi}_m^{-1}\biggr)\bm{D}\nonumber \\
&- \frac{1}{2}\sum_{\bm{x}_n\in\mathcal{X}_m}w_n\bm{D}^\top\left(\bm{\Psi}_m^{-1}\otimes\bm{\Psi}_m^{-1}\tilde{\bm{x}}_n\tilde{\bm{x}}_n^\top\bm{\Psi}_m^{-1}\right)\bm{D} \nonumber \\
&- \frac{1}{2}\sum_{\bm{x}_n\in\mathcal{X}_m}w_n\bm{D}^\top\left(\bm{\bm{\Psi}_m^{-1}\tilde{\bm{x}}_n\tilde{\bm{x}}_n^\top\bm{\Psi}_m^{-1}\otimes\Psi}_m^{-1}\right)\bm{D}.
\label{eq:secdersigma2}
\end{align}
Evaluting Eq.~\eqref{eq:secdersigma2} at the maximum likelihood estimates results in
\begin{align}
\hat{\bm{J}}_{\bm{\Psi}\bm{\Psi}} &= \frac{N_m}{2}\bm{D}^\top(\hat{\bm{\Psi}}_m^{-1}\otimes\hat{\bm{\Psi}}_m^{-1})\bm{D} \nonumber \\
&+ \frac{1}{2(\nu_m+r)}\sum_{\bm{x}_n\in\mathcal{X}_m}w_n^2\bm{D}^\top\biggl(\hat{\bm{\Psi}}_m^{-1}\tilde{\bm{x}}_n\tilde{\bm{x}}_n\hat{\bm{\Psi}}_m^{-1}\nonumber \\
&\otimes\hat{\bm{\Psi}}_m^{-1}\tilde{\bm{x}}_n\tilde{\bm{x}}_n\hat{\bm{\Psi}}_m^{-1}\biggr)\bm{D}\nonumber \\
&-\frac{1}{2}\bm{D}^\top\left(\hat{\bm{\Psi}}_m^{-1}\otimes N_m\hat{\bm{\Psi}}_m^{-1}\right)\bm{D} \nonumber \\
&-\frac{1}{2}\bm{D}^\top\left(N_m\hat{\bm{\Psi}}_m^{-1}\otimes\hat{\bm{\Psi}}_m^{-1}\right)\bm{D} \nonumber \\
&= -\frac{N_m}{2}\bm{D}^\top(\hat{\bm{\Psi}}_m^{-1}\otimes\hat{\bm{\Psi}}_m^{-1})\bm{D} \nonumber \\
&+ \frac{1}{2(\nu_m+r)}\bm{D}^\top\left(\hat{\bm{\Psi}}_m^{-1}\otimes\hat{\bm{\Psi}}_m^{-1}\right)\nonumber \\
&*\sum_{\bm{x}_n\in\mathcal{X}_m}w_n^2\left(\tilde{\bm{x}}_n\tilde{\bm{x}}_n^\top\otimes\tilde{\bm{x}}_n\tilde{\bm{x}}_n^\top\right)\left(\hat{\bm{\Psi}}_m^{-1}\otimes\bm{\Psi}_m^{-1}\right)\bm{D}.
\label{eq:secdersigma3}
\end{align}

In face of Eqs.~\eqref{eq:secdermu3}, \eqref{eq:secdermusigma4} and \eqref{eq:secdersigma3}, three normalization factors exist, which are $\sum_{\bm{x}_n\in\mathcal{X}_m}w_n^2$, $\sum_{\bm{x}_n\in\mathcal{X}_m}w_n$, and $N_m$. While the relationship between $\sum_{\bm{x}_n\in\mathcal{X}_m}w_n^2$ and $\sum_{\bm{x}_n\in\mathcal{X}_m}w_n$ is non-trivial, starting from Eq.~\eqref{eq:psihat} and doing straightforward calculations the authors in \cite{kent1994} showed that
\begin{equation}
 \sum_{\bm{x}_n\in\mathcal{X}_m}w_n = N_m.
\end{equation}
As a result, we end up with only two normalization factors, namely $\sum_{\bm{x}_n\in\mathcal{X}_m}w_n^2$ and $N_m$. Given that $l \ll N$, $N\rightarrow\infty$ indicates that $\epsilon\rightarrow\infty$, where $\epsilon$ is given by Eq.~\eqref{eq:epsilon}. Hence, as $N\rightarrow\infty$ 
\begin{equation}
\bigg|\frac{1}{\epsilon} \hat{\bm{J}}_m\bigg|\approx\mathcal{O}(1),
\label{eq:J5}
\end{equation}
where $\mathcal{O}(1)$ denotes Landau's term which tends to a constant as $N\rightarrow\infty$. Using the result in Eq.~\eqref{eq:J5}, Eq.~\eqref{eq:finalposterior} can be simplified to 
\begin{align}
 \log p(M_l|\mathcal{X}) & \approx \log p(M_l) + \sum_{m=1}^l\log\left( f(\hat{\bm{\theta}}_m|M_l)\mathcal{L}(\hat{\bm{\theta}}_m|\mathcal{X}_m)\right) \nonumber \\
 &+ \frac{lq}{2}\log 2\pi - \frac{1}{2}\sum_{m=1}^l\log \left|\epsilon\frac{\hat{\bm{J}}_m}{\epsilon}\right| -\log f(\mathcal{X}) \nonumber \\
 &= \log p(M_l) + \sum_{m=1}^l\log\left( f(\hat{\bm{\theta}}_m|M_l)\mathcal{L}(\hat{\bm{\theta}}_m|\mathcal{X}_m)\right) \nonumber \\
 &+\frac{lq}{2}\log 2\pi - \frac{q}{2}\!\sum_{m=1}^l\log \epsilon -\frac{1}{2}\!\sum_{m=1}^l\log\left|\frac{\hat{\bm{J}}_m}{\epsilon}\right| \nonumber \\
 &-\log f(\mathcal{X}),
\label{eq:finalposterior2}
\end{align}
where $q=\frac{1}{2}r(r+3)$ is the number of estimated parameters per cluster. 

Assume that
\begin{description}
 \item[$\left(\mathcal{A.}3\right)$] $p(M_l)$ and $f(\hat{\bm{\theta}}_l|M_l)$ are independent of the data length $N$.
\end{description}
Ignoring the terms in Eq.~\eqref{eq:finalposterior2} that do not grow as $N\rightarrow\infty$ results in
\begin{align}
 \text{BIC}_{t_\nu}(M_l) &\triangleq \log p(M_l|\mathcal{X}) \nonumber \\
 &\approx \sum\limits_{m=1}^l\log\mathcal{L}(\hat{\bm{\theta}}_m|\mathcal{X}_m) -\frac{q}{2}\sum_{m=1}^l\log \epsilon -\log f(\mathcal{X}). 
 \label{eq:BICt2}
\end{align}
Substituting the expression of $\log\mathcal{L}(\hat{\bm{\theta}}_m|\mathcal{X}_m)$, given by Eq.~\eqref{eq:loglikeli}, into Eq.~\eqref{eq:BICt2} results in
\begin{align}
 \text{BIC}_{t_\nu}(M_l) 
  &= \sum_{m=1}^l\!N_m\log N_m - \!N\log N - \!\!\sum_{m=1}^l\frac{N_m}{2}\log|\hat{\bm{\Psi}}_m| \nonumber  \\
  & + \sum_{m=1}^lN_m\log\frac{\Gamma\left((\nu_m+r)/2\right)}{\Gamma\left(\nu_m/2\right)(\pi\nu_m)^{r/2}} \nonumber \\
  & -\frac{1}{2}\sum_{m=1}^l\sum_{\bm{x}_n\in\mathcal{X}_m}(\nu_m+r)\log\left(\!1+\frac{\delta_n}{\nu_m}\!\right) \nonumber \\
  &  -\!\frac{q}{2}\sum_{m=1}^l\log \epsilon -\log f(\mathcal{X}).
  \label{eq:BICt3}
\end{align}
Finally, ignoring the model independent terms in Eq.~\eqref{eq:BICt3} results in Eq.~\eqref{eq:BICt}. This concludes the proof. 

\section{Calculation of the Determinant of the Fisher Information Matrix}
\label{app:C}
The Fisher information matrix, given by Eq.~\eqref{eq:FIM1}, is a block matrix and its determinant is calculated as 
\begin{align}
 |\hat{\bm{J}}_m| &= |-\hat{\bm{J}}_{\bm{\mu}\bm{\mu}^\top} + \hat{\bm{J}}_{\bm{\mu}\bm{\Psi}}\hat{\bm{J}}_{\bm{\Psi}\bm{\Psi}}^{-1}\hat{\bm{J}}_{\bm{\Psi}\bm{\mu}^\top}||-\hat{\bm{J}}_{\bm{\Psi}\bm{\Psi}}|,
 \label{eq:est_J}
\end{align}
where $\hat{\bm{J}}_{\bm{\mu}\bm{\mu}^\top}$, $\hat{\bm{J}}_{\bm{\mu}\bm{\Psi}}$, and $\hat{\bm{J}}_{\bm{\Psi}\bm{\Psi}}$ are given by Eqs.~\eqref{eq:secdermu3}, \eqref{eq:secdermusigma4} and \eqref{eq:secdersigma3}, respectively.   

\section{Vector and Matrix Differentiation Rules}
\label{app:d}
Given that $x\in\mathbb{R}^{1\times 1}$, $\bm{x}\in\mathbb{R}^{n\times 1}$, $\bm{y}\in\mathbb{R}^{q\times 1}$, $\bm{z}\in\mathbb{R}^{r\times 1}$, $\bm{X}\in\mathbb{R}^{n\times q}$, $\bm{Y}\in\mathbb{R}^{q\times r}$, and $\bm{Z}\in\mathbb{R}^{r\times m}$,  we have used the following vector and matrix differentiation rules (see \cite{magnus2007} for details):
\begin{align}
\frac{d\left(\bm{x}^\top\bm{x}\right)}{d\bm{x}} &= 2\bm{x}^\top \\
\frac{d\bm{X}^{-1}}{d\bm{X}} &= -\bm{X}^{-1}\frac{d\bm{X}}{d\bm{X}}\bm{X}^{-1} \\
\frac{d\log|\bm{X}|}{d\bm{X}} &= \tr\left(\bm{X}^{-1}\frac{d\bm{X}}{d\bm{X}}\right)  
\end{align}
Following the recommendations in \cite{magnus2010}, the narrow definition of matrix derivative ($\alpha$-derivative), which is defined as 
\begin{equation}
\frac{dF(\bm{X})}{d\bm{X}} = \frac{d\text{vec}\left(F(\bm{X})\right)}{d\text{vec}\left(\bm{X}\right)},
\end{equation}
is used in Appendix~\ref{app:B}. To this end, the following properties of the trace and vec operators are utilized:
\begin{align}
\tr(\bm{X}\bm{Y}) &= \tr(\bm{Y}\bm{X}) \\
\frac{d}{d\bm{X}}\tr\left(\bm{X}\right) &= \tr\left(\frac{d\bm{X}}{d\bm{X}}\right) \\
\frac{d\tr(\bm{X}\bm{Y})}{d\bm{X}}&=\bm{Y} \\
\tr\left(\bm{X}^\top\bm{Y}\right) &= \text{vec}\left(\bm{X}\right)^\top\text{vec}\left(\bm{Y}\right) \\
\text{vec}\left(\bm{x}\right) &= \text{vec}\left(\bm{x}^\top\right) = \bm{x} \\
\frac{d\text{vec}\left(\bm{X}\right)}{d\bm{X}} &= \text{vec}\left(\frac{d\bm{X}}{d\bm{X}}\right) \\
\text{vec}\left(\bm{X}\bm{Y}\right) &= \left(\bm{Y}^\top\otimes\bm{I}_n\right)\text{vec}\left(\bm{X}\right) \nonumber \\
& = \left(\bm{I}_r\otimes\bm{X}\right)\text{vec}\left(\bm{Y}\right) \\
\text{vec}\left(\bm{X}\bm{Y}\bm{Z}\right) &= \left(\bm{Z}^\top\otimes\bm{X}\right)\text{vec}\left(\bm{Y}\right) \\
\frac{d}{d\bm{X}}\left(\text{vec}(\bm{X}\otimes\bm{Y}\right) &= \left(\bm{I}_q\otimes \bm{K}_{r, n}\otimes\bm{I}_q\right) \nonumber \\
&*\biggl[\left(\bm{I}_{nq}\otimes\text{vec}(\bm{Y})\right)\frac{d\text{vec}(\bm{X})}{d\bm{X}} \nonumber \\
& +  \left(\text{vec}(\bm{X})\otimes\bm{I}_{qr}\right)\frac{d\text{vec}(\bm{Y})}{d\bm{X}}\biggr] 
\end{align}
where $\bm{K}_{r,n}\in\mathbb{R}^{rn\times rn}$ is the commutation matrix which possess the following properties:
\begin{align}
\bm{K}_{n,n} &= \bm{K}_{n} \\
\bm{K}_n\bm{D}_n &= \bm{D}_n \\
\left(\bm{X}\otimes\bm{y}^\top\right)\bm{K}_{q} &= \left(\bm{y}^\top\otimes\bm{X}\right).
\end{align}
$\bm{D}_n\in\mathbb{R}^{n^2\times\frac{1}{2}n(n+1)}$ is the duplication matrix. 

In addition, the following properties of the Kronecker product are used in the derivations:
\begin{align}
x(\bm{X}\otimes\bm{Y}) &= (x\bm{X})\otimes\bm{Y} = \bm{X}\otimes (x\bm{Y}) \\
\sum_{n=1}^{N}\left(\bm{X}\otimes\bm{Y}_n\right) &= \bm{X} \otimes\sum_{n=1}^{N}\bm{Y}_n 
\end{align}

\section*{Acknowledgment}
We would like to thank Christian Schroth for reviewing the derivations and providing feedback. 

\bibliographystyle{IEEEtran}
\bibliography{IEEEabrv,myReferences}

\end{document}